\documentclass[final,onefignum,onetabnum]{siamonline250211}

\usepackage{lipsum}
\usepackage{amsfonts}
\usepackage{graphicx}
\usepackage{epstopdf}
\usepackage{algorithmic}
\usepackage{cleveref}
\ifpdf
  \DeclareGraphicsExtensions{.eps,.pdf,.png,.jpg}
\else
  \DeclareGraphicsExtensions{.eps}
\fi

\usepackage{enumitem}
\setlist[enumerate]{leftmargin=.5in}
\setlist[itemize]{leftmargin=.5in}

\newsiamremark{remark}{Remark}

\newsiamthm{setup}{Setup}
\crefname{setup}{Setup}{Setups}

\AddToHook{env/setup/begin}{\crefalias{theorem}{setup}}

\headers{Denoising optimized compressive sampling}{Y. Plan, M. S. Scott, X. Sheng and O. Yilmaz}

\title{Denoising guarantees for optimized sampling schemes in compressed sensing\thanks{Submitted for review to the SIAM Journal on Mathematics of Data Science (SIMODS).}}
\author{Yaniv Plan\thanks{\noindent Department of Mathematics, University of British Columbia, Vancouver, BC, Canada
(\email{matthewscott@math.ubc.ca}, \email{xsheng@math.ubc.ca}, \email{yaniv@math.ubc.ca}
\email{oyilmaz@math.ubc.ca})
\newline 
{\bf Author roles:} Authors listed in alphabetic order.  MS was primarily responsible for developing the theory, the sparsity-based numerics, and writing the paper. XS was primarily responsible for training the generative model and creating and presenting the related numerics.
}
\and Matthew S. Scott\footnotemark[1]
\and  Xia Sheng\footnotemark[1]
\and Ozgur Yilmaz\footnotemark[1] \thanks{CNRS -- PIMS International Research Laboratory}}

\usepackage{amsopn}

\usepackage{mathtools}
\usepackage{amssymb}

\newcommand{\allone}{\mathbf{1}}
\newcommand{\trunc}{\mathbb{T}}

\newcommand{\sphere}[1]{\mathbb{S}^{#1-1}}

\newcommand{\field}{\mathbb{R}}
\newcommand{\measfield}{\mathbb{K}}

\newcommand{\reals}{\mathbb{R}}

\newcommand{\complex}{\mathbb{C}}

\newcommand{\iid}{\overset{\text{iid}}{\sim}}

\newcommand{\minimize}{\mathop{\mathrm{minimize}}}
\DeclareMathOperator{\Diag}{Diag}

\DeclareMathOperator{\range}{\mathrm{range}}

\DeclareMathOperator{\Span}{\mathrm{span}}

\DeclareMathOperator{\supp}{\mathrm{supp}}

\DeclareMathOperator{\proj}{\Pi}

\ifpdf
\hypersetup{
  pdftitle={Denoising guarantees for optimized sampling schemes in compressed sensing},
  pdfauthor={Y. Plan, M. S. Scott, X. Sheng,  and O. Yilmaz}
}
\fi
\begin{document}
\maketitle
\begin{abstract}
Compressed sensing with subsampled unitary matrices benefits from 
\emph{optimized} sampling schemes, which feature improved theoretical guarantees and empirical performance
relative to uniform subsampling. We
provide, in a first of its kind in compressed sensing, theoretical guarantees
showing that the error caused by the measurement noise vanishes with an
increasing number of measurements for optimized sampling schemes, assuming that the noise is Gaussian. We
moreover provide similar guarantees for measurements sampled with-replacement
with arbitrary probability weights. All our results hold on prior sets
contained in a union of low-dimensional subspaces. Finally, we demonstrate that this
denoising behavior appears in empirical experiments with a rate that closely
matches our theoretical guarantees when the prior set is the range of a generative ReLU neural network and when it is the set of sparse vectors.
\end{abstract}


\section{Introduction}
\label{loc:body.introduction}
The field of compressed sensing considers signals
$\boldsymbol{x}_0 \in \field^n$ with high ambient dimension $n$ that belong to
(or can be well-approximated by) a prior set $\mathcal{Q} \subseteq \field^n$
with much lower effective dimensionality than the ambient dimension. 
The aim is to recover such signals with provable accuracy guarantees
from noisy measurements of the form
\begin{displaymath}
\boldsymbol{b} = A\boldsymbol{x}_0 +\boldsymbol{\eta},
\end{displaymath}
where
$\boldsymbol\eta \in \measfield^m$ is the noise and $A \in \measfield^{m \times n}$
is the CS matrix with $m \ll n$. Here , $\mathbb{K}$ denotes a field, which may be either $\mathbb{R}$ or $\mathbb{C}$. A fundamental question concerns both stability and robustness: how many measurements
$m$ are required to guarantee stable recovery of the unknown signal $\boldsymbol{x}_0$
if it is known that $\boldsymbol{x}_0$ is close to the prior set $\mathcal{Q}$?
Additionally, how robust is the recovery in the presence of noise -- specifically, how does the noise level impact the accuracy of the recovery?

In this work, we consider the case where the prior set $\mathcal{Q}$ is contained in a union
of low-dimensional subspaces. This general model includes classical sparsity-based priors as well as the
more recently introduced generative priors, where $\mathcal{Q}$ is the range of
a trained generative neural network with ReLU activations~\cite{boraCompressedSensingUsing2017}.

\subsection*{Structured compressed sensing and variable-density sampling}
Compressed sensing matrices with subsampled unitary structure have been
foundational in compressed sensing since its early development~\cite{vershyninHighDimensionalProbabilityIntroduction2018, candesRobustUncertaintyPrinciples2006, foucartMathematicalIntroductionCompressive2013}. Certain
structured CS matrices appropriately model magnetic resonance
imaging (MRI), as physical constraints dictate that measurements must
be taken in the Fourier domain~\cite{lustigSparseMRIApplication2007}. Moreover, certain structured matrices offer significant computational advantages, e.g., fast transforms, reduced storage requirements, and the ability to perform matrix free operations~\cite{khvostovaAlgorithmMachineCalculation1965}. This motivates the study
of CS matrices of the form $A = SF$ where $F \in
\measfield^{n \times n}$ is an appropriate unitary matrix  such as the Discrete Fourier Transform (DFT) matrix, 
and $S \in \mathbb{R}^{m \times n}$ is a sampling matrix that selects $m$
rows of $F$ to be included in the CS matrix, sometimes
repeatedly. We call the rows of $F$ \emph{measurement vectors}, and together they form the
\emph{measurement basis}.

In compressed sensing, it was recognized early on that not all measurements
carry equal importance for signal recovery.
For example, low Fourier frequencies exhibit strong correlation with natural
images, and thus are more informative~\cite{adcockBreakingCoherenceBarrier2017a}.
This observation led to the development of {\it variable density sampling strategies},
where certain measurements are acquired more frequently than others based
on their relative importance~\cite{candesSparsityIncoherenceCompressive2007, puyVariableDensityCompressive2011, krahmerStableRobustSampling2014}.
A theoretical framework that captures the notion of a measurement vector being ``informative" is that of {\it local coherence},
which quantifies the degree of alignment of individual measurement vectors
with the prior set (see \Cref{loc:local_coherence.statement})~\cite{krahmerLocalCoherenceSampling2013}.

The theory of variable density sampling relies on a preconditioning matrix to improve the condition number of the measurement matrix on the prior set. Preconditioning in this context was first introduced in dictionary learning \cite{schnassDictionaryPreconditioningGreedy2008} and later in compressed sensing for sparse polynomial approximations \cite{rauhutSparseLegendreExpansions2012}. With a diagonal preconditioner $D$, the rows of the preconditioned CS matrix $SDF$ form jointly-isotropic vectors \cite{adcockCompressiveImagingStructure2021}. Alternatively, preconditioned subsampled unitary matrices can be seen as bounded orthonormal systems with respect to an orthogonalization measure, where the preconditioner serves as the measure itself \cite{krahmerStableRobustSampling2014, rauhutCompressiveSensingStructured2010}.

\subsection*{Optimized sampling schemes}
Suppose that we have a general sample complexity bound that applies to a parametric
family of sampling distributions.
Given a prior set and measurement vectors in a unitary matrix
$F$ with known local coherences, we can then identify the ``best"
distribution as the one that minimizes this general sample complexity bound.

We call such sampling distributions \emph{optimized} sampling distributions, also known as \emph{optimal variable-density} sampling in the literature. This idea was pioneered by Candes and Romberg~\cite{candesSparsityIncoherenceCompressive2007}, who used an optimized sampling distribution for Fourier measurements on a Haar wavelet sparsity basis. This approach was later formalized theoretically in compressed sensing by Puy, Vandergheynst, and Wiaux~\cite{puyVariableDensityCompressive2011}, who framed the optimized sampling distribution as the sampling distribution minimizing general sample complexity bounds for arbitrary sampling probability weights.

In works preceding variable density sampling, a single scalar coherence
parameter was used for the full measurement basis, e.g.,~\cite{rudelsonSparseReconstructionFourier2008, candesSparsityIncoherenceCompressive2007},
corresponding to the maximum of all the local coherences of the measurement
vectors with respect to the sparsity basis. An analogous parameter was introduced
in~\cite{berkModeladaptedFourierSampling2023} for the generative setting,
considering the maximum of the local coherences with respect to the full prior
set.
This coherence parameter can be derived from the present work when the
sampling probabilities are uniform (see, e.g., \Cref{loc:with:replacement_complexity.statement}). 

This earlier coherence parameter was originally introduced to assess the suitability of a measurement basis for a given prior set rather than to inform sampling. In compressed sensing, subsampled unitary measurement matrices can be unreliable in observing directions highly aligned with the measurement basis, as such directions may fall into the null space of the sensing matrix when the corresponding measurement vectors are not sampled~\cite{candesSparsityIncoherenceCompressive2007, donohoCompressedSensing2006}. This phenomenon, closely tied to the null space property~\cite{cohenCompressedSensingBest2009}, highlights the role of incoherence: when a measurement basis has a small coherence parameter, it ensures that differences between signals in the prior set are misaligned with measurement vectors, thereby enabling their reliable differentiation~\cite{candesNearOptimalSignalRecovery2006, foucartMathematicalIntroductionCompressive2013}. In other words, a measurement basis is well suited to a prior set when its coherence parameter is small, particularly in the uniform sampling setting, as this prevents signal components from being indistinguishable under the measurement process.
This notion of \emph{fitness} of a measurement basis relative to a prior set can be significantly relaxed when
we have control over the sampling distribution. This expands the set of acceptable measurement
bases to include important additional cases, such as the discrete Fourier basis when the signal is sparse in the Haar wavelet basis. Specifically, for a fixed prior set and a measurement basis with local coherence vector
$\boldsymbol{\alpha}$ (as defined in \Cref{loc:local_coherence.statement}), optimizing the sampling scheme reduces the sample complexity from $n \|\boldsymbol{\alpha}\|_{\infty}^2$---a sample complexity found in~\cite{berkCoherenceParameterCharacterizing2022}---to
$\|\boldsymbol{\alpha}\|_2^2$, 
as shown in \Cref{loc:optimized_with:replacement_cs_on_union_of_subspaces.statement}. This gap is significant when there is a large variation in the local coherences, as is the case in our numerical experiments: see \Cref{fig:es_vs_wr}. 
From this perspective, the theory of optimized sampling schemes extends compressed sensing with subsampled unitary CS matrices
to a substantially broader class of signal recovery problems.

\subsection*{Noise models and denoising in compressed sensing}
Another key aspect of the theory of
compressed sensing is  robustness of signal recovery to noise in
the measurements~\cite{candesDecodingLinearProgramming2005}, which has been an integral challenge in
compressive signal recovery, and is
part of the reason why specialized tools such as the Restricted Isometry
Property (RIP) were developed .

There are a few types of noise that appear in the literature. Noise can be modeled
\emph{deterministic}, which means that it is an arbitrary fixed vector
(e.g.,~\cite{berkModeladaptedFourierSampling2023, adcockUnifiedFrameworkLearning2023, naderiIndependentMeasurementsGeneral2021}),
in which case the accuracy of the signal recovery typically depends on the norm
of the noise vector. There is also \emph{bounded} or \emph{adversarial}
noise (e.g., \cite{krahmerStableRobustSampling2014}), where the noise is a random vector
that may depend (adversarially) on the random CS matrix, but which has bounded
norm, or alternatively, that is
constrained to lie in a specified set. Finally, 
there is the model 
we consider: independent stochastic
noise, where the noise is random
and independent of the CS matrix. We focus on independent Gaussian noise
and show that taking additional measurements can denoise the signal reconstruction. In other words, the error caused by the noise decreases as $m$ increases, and with a sufficiently large number of measurements the effect of the noise can be made arbitrarily small.

The main contribution of this paper is to describe the denoising effect in
variable and optimized sampling schemes. While it is straightforward to find results for Gaussian noise by using results with deterministic noise, results obtained by such methods do not exhibit denoising properties,
as we discuss in~\Cref{loc:appendix.gaussian_noise_corollary_from_deterministic_noise}.

Most existing work on signal recovery from noisy measurements in compressed sensing has focused on subgaussian CS matrices. Sub-sampled unitary matrices have large subgaussian norm and thus require specialized machinery.  This is made more delicate when there is variation in the
local coherences.
Measurement vectors with higher local coherence tend to capture more signal
energy, making their measurements both more informative and more robust to noise. 
In contrast, measurements produced by measurement vectors with low local coherence are more susceptible to
noise, as the signal component is often weaker relative to the noise.
Capturing this effect in theoretical analysis is challenging because 
robustness depends on which measurement vectors are randomly selected. 
Previous results that contend with this difficulty with either deterministic or adversarial noise have either made unrealistic assumptions or given difficult-to-apply error bounds, as we argue in 
\Cref{loc:body.noise_robustness_in_the_literature}.
By considering Gaussian noise, we provide simple and meaningful error bounds without the drawbacks associated with other types of noise.

Similar works
include~\cite{cardenasCS4MLGeneralFramework2023, berkModeladaptedFourierSampling2023},
where optimized sampling was introduced to the generative setting. The theme is
revisited in~\cite{adcockUnifiedFrameworkLearning2023}, where priors that are
contained in finite unions of subspaces are considered. Their results, which we
further discuss in \Cref{loc:body.noise_robustness_in_the_literature},
assume deterministic noise.

\textbf{Contributions}

\label{loc:introduction_denoising_paper.body.contributions}
\begin{itemize}
\item In a first for compressed sensing, in \Cref{loc:optimized_with:replacement_cs_on_union_of_subspaces.statement}
we derive denoising results for
optimized sampling schemes. We show that the error induced by noise decreases proportionally to $1/\sqrt{m}$ where $m$ is the number of measurements.
\item More generally, in \Cref{loc:cs_with_replacement_and_with_denoising_on_unions_of_subspaces.statement}
we provide a denoising compressed sensing result for  
arbitrary \emph{variable density} sampling schemes.
\item Our results hold for priors contained in finite unions of subspaces including both sparse and generative priors. We discuss this point in \Cref{loc:body.denoising_for_subsampled_unitary_matrices.applications}.
\item Our results hold for prior sets in $\mathbb{R}^n$, and CS matrices that are either real or complex. We carefully consider the effect of complex measurements on real subspaces, and find that the same results hold regardless of the field, despite defining the noise to have expected squared norm twice as large in the complex case.
\end{itemize}

\textbf{Notation}

\label{loc:body.introduction.notation}
Let $\mathbb{R}_+$ be the non-negative real numbers, $\mathbb{R}_{++}$
the strictly positive real numbers, and $\mathbb{N}$ the natural numbers
starting at $1$. For a function $f$, we denote its range by $\range(f)$,
and its restriction to a subset $C$ of its domain by $f|_C$.
Throughout this paper, we fix the field $\mathbb{K}$ to be either $\mathbb{C}$ or
$\mathbb{R}$. 

For a vector $\boldsymbol{u}$, its components are indexed as $u_i$.
We denote by
$\{\boldsymbol{e}_i\}_{i  \in [n]}$ the canonical basis of $\mathbb{R}^n$.
For $\ell \in \mathbb{N}$, the set $[\ell]$ comprises
the integers from $1$ to $\ell$. 
We denote by $\supp \boldsymbol{v}$ the support of $\boldsymbol{v}$, and by
$\boldsymbol{v}^{.2}$ entry-wise square of $\boldsymbol{v}$.

For an $m  \times  n$ matrix $A$, we denote its adjoint (the conjugate
transpose) by
$A^*$, its entries by
$A_{i,j}$, and its row vectors by $\boldsymbol{a}_i$,
such that $A  =  \sum_{i = 1}^m \boldsymbol{e}_i \boldsymbol{a}_i^*$.
The Euclidean norm
of a vector $\boldsymbol{u} \in \measfield^n$ is
$\|\boldsymbol{v}\|_2 := \sqrt{\boldsymbol{v}^* \boldsymbol{v}}$.
The
operator norm of a matrix $A$ is $\|A\|:= \sup_{\boldsymbol{u} \in B_2^n} \|Au\|_2$.
For matrices, given a vector $\boldsymbol{d}  \in \mathbb{R}^n$, we
denote by $\Diag(\boldsymbol{d})$ the $n  \times n$ diagonal
matrix with diagonal entries $\boldsymbol{d}$. The identity matrix in
$\mathbb{R}^m$ is labeled $I_m$.
Projection onto a closed set $\mathcal{T} \subseteq \mathbb{R}^n$ is denoted by $\proj_{\mathcal{T}}$, mapping a vector
$\boldsymbol{x}$ to the element in $\mathcal{T}$ that minimizes the Euclidean
distance, with ties broken by choosing the lexicographically first (meaning
that vectors are ordered by
their first
entry, then second, then third, and so on).

We use $\langle \cdot, \cdot \rangle$ to denote
the inner product in $\mathbb{K}^n$; specifically, the canonical inner product
when $\measfield$ is $\mathbb{R}$, and the complex inner
product $\langle \boldsymbol{u}, \boldsymbol{v}\rangle = \boldsymbol{u}^* \boldsymbol{v}$ when $\measfield$ is $\mathbb{C}$. We also denote by
$\mathcal{R}\langle  \cdot ,  \cdot \rangle$ the real part of the inner
product (which is just the canonical inner product when $\measfield$ is $\mathbb{R}$).

$\sphere{n}$ is the unit sphere in $\mathbb{R}^n$ or $\mathbb{C}^n$ depending
on context. 
The simplex $\Delta^{n-1}$ is defined as:
\begin{displaymath}
\Delta^{n-1} = \left\{ \boldsymbol{p} \in \mathbb{R}^n \mid p_i \geq 0, \sum p_i = 1 \right\}.
\end{displaymath}
We let $B_2$ be the $\ell_2$ ball, and $B_2^n$ be the
$l_2$ ball of dimension $n$ specifically. We say that a set $\mathcal{T}$ in
a real or complex vector space is a \emph{cone} when $\forall \lambda  \in (0, \infty), \lambda \mathcal{T} = \mathcal{T}$, where $\lambda \mathcal{T} := \{\lambda t | t  \in \mathcal{T}\}$.

The self-difference $\mathcal{V} - \mathcal{V}$ is $\left\{ \boldsymbol{v}_1 - \boldsymbol{v}_2 \mid \boldsymbol{v}_1, \boldsymbol{v}_2 \in \mathcal{V} \right\}$. Denote by $\mathcal{P}(\measfield^n)$ the powerset of $\measfield^n$.

We employ the notation $a \lesssim b$ if $a \leq Cb$ where $C$ is an absolute
constant, potentially different for each instance. 

We denote $X \sim \mathcal{N}(\mu, \sigma^2)$ to be the Gaussian random variable
with mean $\mu$ and variance $\sigma^2$. A random Gaussian vector $g \sim \mathcal{N}(0, I_m)$ is a random vector in $\mathbb{R}^m$ which has i.i.d.
$\mathcal{N}(0, 1)$ Gaussian entries. A complex random Gaussian vector 
$g \sim \mathcal{N}(0, I_m)$, $g  \in \mathbb{C}^m$, is a random
vector in $\mathbb{C}^m$ with entries that have real and imaginary parts individually $\iid \mathcal{N}(0, 1)$.
\section{Main result}
\label{loc:body.main_result}
We introduce key mathematical objects that will be of use in our main result.
\begin{definition}[Local coherence]
\label{loc:local_coherence.statement}
The \emph{local coherence} of a vector $\boldsymbol{\phi} \in \measfield^n$ with respect to a cone $\mathcal{T} \subseteq \measfield^n$ is defined as
\begin{displaymath}
\alpha_{\mathcal{T}}(\boldsymbol\phi) := \sup_{\boldsymbol{x} \in \mathcal{T} \cap B_2} \lvert \boldsymbol\phi^* \boldsymbol{x} \rvert.
\end{displaymath}
The \emph{local coherences} of a unitary matrix $F \in \measfield^{n \times n}$ with respect to a cone $\mathcal{T} \subseteq \measfield^n$ is the vector $\boldsymbol{\alpha} \in \mathbb{R}^n_{+}$ with entries $\alpha_j := \alpha_\mathcal{T}(\boldsymbol{f}_j)$, where $\boldsymbol{f}_j$ is the $j^{th}$ row of $F$.
\end{definition}
Note that we define the local coherences with respect to some cone 
$\mathcal{T}  \subseteq \mathbb{K}^n$, yet we also have local coherences
of complex vectors with respect to cones in $\mathbb{R}^n$ by embedding
these cones into $\mathbb{C}^n$ in the canonical way.
\begin{remark}
\label{loc:local_coherence.remark_union_subspaces}
Local coherences in this paper are restricted to
be strictly positive, i,e., the
rows of $F$ may not be orthogonal to the cone
$\mathcal{T}$. We make this assumption for simplicity. To allow
$F$ to have fully incoherent rows, one need only let the sampling probabilities
of orthogonal rows be $0$, and apply our results to the subspace spanned by
the remaining rows.
\end{remark}

Recall that we consider CS matrices of the form $A  =  SF$, where $F$ is a $n  \times n$ unitary matrix (for example, $F$ can be the Fourier
matrix), and $S$ is a \emph{sampling matrix}, which we now
define.
\begin{definition}[Sampling matrix]
\label{loc:sampling_matrix.with_replacement_deterministic}
A \emph{sampling matrix} $S \in \mathbb{R}^{m  \times n}$ for $m,  n  \in \mathbb{N}$ is a matrix which has rows of the form $\sqrt{\frac{n}{m}} \boldsymbol{e}_i$, for any $i  \in [n]$.
\end{definition}
A sampling matrix $S$ will typically be random in the following way.
\begin{definition}[With-replacement sampling matrix]
\label{loc:sampling_matrix_with_replacement.statement}
A \emph{with-replacement sampling matrix} $S \in \mathbb{R}^{m \times n}$ associated with a sampling probability vector $\boldsymbol{p} \in (0, 1)^n \cap \Delta^{n-1}$ is the matrix with i.i.d. row vectors $\{ \boldsymbol{s}_i \}_{i \in [m]}$ such that 
\begin{displaymath}
\mathbb{P}\left( \boldsymbol{s}_i = \sqrt{ \frac{n}{m} }\boldsymbol{e}_j \right)= p_j \quad \forall i \in [m],j \in [n].
\end{displaymath}
\end{definition}
With these objects, we outline the problem of robust signal recovery in greater detail.
\begin{setup}
\label{loc:setup_of_signal_recovery_with_subsampled_unitary_measurements_and_gaussian_noise.for_denoising_paper}
\,
\normalfont

\textbf{Prior and true signal.} Let $\boldsymbol{x}_{0} \in \field^n$ be a signal, and 
$\mathcal{Q} \subseteq \field^n$ be the prior set, which models the signals of interest. As such, we think of $\boldsymbol{x}_0$ as being close to $\mathcal{Q}$;  $\boldsymbol{x}^\perp := \boldsymbol{x}_0 - \proj_{\mathcal{Q}} \boldsymbol{x}_0$ quantifies the model mismatch. 
\smallskip 

\textbf{Measurement acquisition.} Let $F  \in \measfield^{n  \times n}$ be a unitary matrix. Let $\mathcal{T} \subseteq \mathbb{R}^n$ be a union of $M$
subspaces, each of dimension at most $\ell$, such that
$\mathcal{T} \supseteq \mathcal{Q}-\mathcal{Q}$. Suppose that $\boldsymbol{\alpha}$
is the vector of local coherences of $F$ with respect to
$\mathcal{T}$.

Let $S$ be a possibly random
\emph{sampling matrix}, and define the measurements
\begin{displaymath}
\boldsymbol{b} = SF\boldsymbol{x}_{0} + \boldsymbol{\eta},
\end{displaymath}
where the noise is $\boldsymbol{\eta} =  \frac{\sigma \boldsymbol{g}}{\sqrt{m}}$
with 
$\boldsymbol{g} \sim \mathcal{N}(0, I_m)$ being a Gaussian vector in $\measfield^m$. Here,
$\mathbb{E}[\|\boldsymbol{\eta}\|_2^2]$ is $\sigma^2$ when $\mathbb{K}$
is $\mathbb{R}$ and $2 \sigma^2$ when $\mathbb{K}$ is $\mathbb{C}$,  thus, $\sigma$
determines the size of the noise. With this normalization, $\frac{1}{\sigma}$
can be thought of as the Signal-to-Noise Ratio (SNR) up to an absolute constant. Specifically, it is the SNR in expectation if $\boldsymbol{x}_0$ is a random signal on the sphere.

\smallskip 
\textbf{Signal reconstruction} Knowing only $\boldsymbol{b}, S$ and $F$, we (approximately) recover the true signal $\boldsymbol{x}_0$ by (approximately) solving the following optimization problem:
\begin{equation}
\label{eq:opt:reconstruct}
\minimize_{\boldsymbol{x} \in \mathcal{Q}} \, \lVert \widetilde{D} SF\boldsymbol{x}-\widetilde{D}\boldsymbol{b} \rVert_{2}^2\end{equation}
where $\widetilde{D} := \Diag(S\boldsymbol{d})$ is a diagonal preconditioning
matrix for some $\boldsymbol{d} \in \mathbb{R}^n$.
Note that in terms of $D := \Diag(\boldsymbol{d})$, the preconditioned CS
matrix $\widetilde{D}SF$ can be written as $SDF$. This demonstrates that the
preconditioning is, in fact, an 
element-wise scaling operation applied to individual rows of $F$.
We approximately solve the optimization problem \eqref{eq:opt:reconstruct} and obtain an $\hat{\boldsymbol{x}}\in \mathcal{Q}$ such that 
\begin{equation}
\label{eq:opt:recov}
\lVert  \widetilde{D}SF\hat{\boldsymbol{x}}-\widetilde{D}\boldsymbol{b} \rVert_{2}^2 \leq  \min_{\boldsymbol{x} \in \mathcal{Q}}\lVert  \widetilde{D}SF\boldsymbol{x} - \widetilde{D}\boldsymbol{b} \rVert_2^2+\varepsilon
\end{equation}
for some small optimization error $\varepsilon>0$.

\smallskip 
\textbf{Error bounds.} We bound the error 
$\|\boldsymbol{x}_0- \hat{\boldsymbol{x}}\|_2$
in terms of the noise level $\sigma$, the optimization error $\varepsilon$,
and the distance of the true signal $\boldsymbol{x}_0$ to the prior
$\mathcal{Q}$ (the approximation error). 

\end{setup}
Our first theorem concerns \emph{optimized} sampling schemes, which we now define.
\begin{definition}[Optimized sampling vector]
\label{loc:optimized_sampling_vector.statement}
Given a local coherence vector $\boldsymbol{\alpha}$, we define the \emph{optimized
sampling vector} $\boldsymbol{p}' \in \Delta^{n-1} \cap (0,1)^n$ by 
\begin{displaymath}
p'_i = \frac{\alpha_i^2}{\|\boldsymbol{\alpha}\|_2^2}
\end{displaymath}
for each component of $\boldsymbol{p}'$.
\end{definition}
This sampling vector optimizes the sample complexity as a function of the probability vector, see~\Cref{loc:optimize_the_with:replacement_sampling_probabilities.statement}. We now state our main result, which provides denoising signal recovery guarantees for optimized sampling
schemes.
\begin{theorem}[Compressed sensing with optimized sampling]
\label{loc:optimized_with:replacement_cs_on_union_of_subspaces.statement}
Under~\Cref{loc:setup_of_signal_recovery_with_subsampled_unitary_measurements_and_gaussian_noise.for_denoising_paper}, let 
$\delta>0$ and suppose that
\begin{displaymath}
m \gtrsim \lVert \boldsymbol{\alpha}\rVert_{2}^2 \left( \log \ell + \log M + \log \frac{1}{\delta} \right).
\end{displaymath}
Furthermore, let $S$ be an $m \times n$ \hyperref[loc:sampling_matrix_with_replacement.statement]{with-replacement sampling matrix} governed by the optimized probability vector $\boldsymbol{p}'$, and define $D := \Diag\left(\boldsymbol{d}\right)$ where $d_i = (n p'_i)^{- 1 / 2}$. Then the following holds with probability at least $1-\delta$.

For any $\boldsymbol{x}_0  \in \field^n$, with $\varepsilon, \hat{\boldsymbol{x}}, \boldsymbol{x}^{\perp}$ as in~\Cref{loc:setup_of_signal_recovery_with_subsampled_unitary_measurements_and_gaussian_noise.for_denoising_paper}, we have 
\begin{align*}
\lVert \hat{\boldsymbol{x}}- \boldsymbol{x}_0\rVert_2 \leq 9\frac{ \sigma}{\sqrt{ m }} \lVert \boldsymbol{\alpha}\rVert_{2} \min\left(\sqrt{\frac{5}{4\delta}}, \frac{1}{\sqrt{n} \min(\boldsymbol{\alpha})}\right) \left( \sqrt{ \ell } + \sqrt{\log M} + \sqrt{ \log \frac{20}{\delta}}\right)&\\
+\lVert \boldsymbol{x}^\perp\rVert + 6\lVert SDF\boldsymbol{x}^\perp\rVert_{2}  + \frac{3}{2}\sqrt{ \varepsilon }.
\end{align*}
\end{theorem}
\Cref{loc:optimized_with:replacement_cs_on_union_of_subspaces.statement}
follows from \Cref{loc:rip_of_with:replacement_non:uniform_sampling_matrix_on_a_union_of_subspaces.statement}, \Cref{loc:signal_recovery_with_subsampled_unitary_matrix_with_gaussian_noise_on_union_of_subspaces.statement}, and \Cref{loc:tail_on_the_noise_sensitivity_for_adapted_with:replacement_sampling.statement} with a union bound.
See \hyperlink{loc:optimized_with:replacement_cs_on_union_of_subspaces.proof}{the proof} in~\Cref{loc:appendix.proofs_by_union_bounds} for details on how to take the union bound.
\begin{remark}
\label{loc:optimized_with:replacement_cs_on_union_of_subspaces.remark}
Note that the second ``model mismatch" or ``approximation" error term, $6\|SDF\boldsymbol{x}^\perp\|_2$,  in~\Cref{loc:optimized_with:replacement_cs_on_union_of_subspaces.statement} may be absorbed into the first error term $\|\boldsymbol{x}^\perp\|_2$ for a weaker, non-uniform version of our result (i.e., of the form: for any fixed $\boldsymbol{x}_0  \in \mathbb{R}^n$, with high probability on $S$ and the noise, $\|\hat{\boldsymbol{x}}- \boldsymbol{x}_0\|_2  \le  \ldots$). Indeed, in such a case, $\mathbb{E}\|SDF\boldsymbol{x}^\perp\|_2 \le   \|\boldsymbol{x}^\perp\|_2$, and a high-probability bound is given by Markov's inequality. For further discussion on the optimality and typical magnitude of this error term, see~\cite[Section S4]{berkCoherenceParameterCharacterizing2022}. Note also that generally, non-uniform results follow directly from their uniform counterparts.
\end{remark}
\begin{remark}
\label{loc:optimized_with:replacement_cs_on_union_of_subspaces.simplified_result_remark}
\Cref{loc:optimized_with:replacement_cs_on_union_of_subspaces.statement} implies the following. Let $m  \gtrsim \|\boldsymbol{\alpha}\|_2 (\log \ell + \log M + 1)$ and $\boldsymbol{x}_0  \in \mathbb{R}^n$ be a fixed vector. Then with probability at least $0.99$, any approximate minimizer $\hat{\boldsymbol{x}}$ as described in \Cref{eq:opt:reconstruct} will satisfy 
\begin{displaymath}
\|\hat{\boldsymbol{x}}-\boldsymbol{x}_0\|_2 \lesssim \frac{\sigma}{\sqrt{m}} \|\boldsymbol{\alpha}\|_2 ( \sqrt{\ell} + \sqrt{\log M}) + \|\boldsymbol{x}^\perp\|_2 + \sqrt{\varepsilon}.
\end{displaymath}
\end{remark}
\section{Denoising for subsampled unitary matrices}
\label{loc:body.denoising_for_subsampled_unitary_matrices}
To help control various quantities in our analysis, we introduce the unit
truncation operator, which truncates a vector to its leading entries to obtain
a vector of unit norm. Below, given a vector $\boldsymbol{v} \in \mathbb{R}^n$ and a positive integer $s \leq n$, 
$\boldsymbol{v}|_{[s]} \in \mathbb{R}^s$ denotes $\boldsymbol{v}$ truncated to indices $j \le s$.
\begin{definition}[Unit truncation]
\label{loc:unit_truncation.statement}
Given some $\boldsymbol{v} \in \mathbb{R}^n$, let
\begin{displaymath}
I =\min \left\{ \bar{I} \in  [n] \middle| \|v|_{[\bar{I}]}\|_2 \ge 1 \right\}.
\end{displaymath}
Then define the unit truncation operator $\trunc: \measfield^n  \to  \measfield^n$ to have
entries
\begin{displaymath}
\trunc(\boldsymbol{v})_i := \begin{cases}
v_i & i < I, \\
\sqrt{1-\|\boldsymbol{v}|_{[I-1]}\|_2^2} & i  = I,  \\
0 & i> I.
\end{cases}
\end{displaymath}
\end{definition}

Next we recall the general version of the celebrated \textit{restricted isometry property} 
introduced by Candes, Romberg, and Tao~\cite{candesRobustUncertaintyPrinciples2006}.
\begin{definition}[Restricted Isometry Property]
\label{loc:rip.statement}
Let $\mathcal{T} \subseteq \field^n$ be a cone and $A \in \complex^{m \times n}$ a matrix. We say that $A$ satisfies the \emph{Restricted Isometry Property} (RIP) on $\mathcal{T}$ when
\begin{displaymath}
\sup_{u \in \mathcal{T} \cap \mathcal{S}^{n-1}}\lvert \lVert Au\rVert_2 - 1\rvert \leq \frac{1}{3}.
\end{displaymath}
\end{definition}
Such a property is desirable because a measurement matrix with the RIP ensures that distinct signals within the prior set are mapped to sufficiently different measurements provided the noise is small enough.

We next show that both signal recovery and denoising occurs for any subsampled
measurement matrix that satisfies the RIP when preconditioned. Later in \Cref{loc:rip_of_with:replacement_non:uniform_sampling_matrix_on_a_union_of_subspaces.statement} we prove that under suitable conditions, such an RIP holds with high probability
when $S$ is a with-replacement random sampling matrix.
\begin{theorem}[Signal recovery]
\label{loc:signal_recovery_with_subsampled_unitary_matrix_with_gaussian_noise_on_union_of_subspaces.statement}
Under~\Cref{loc:setup_of_signal_recovery_with_subsampled_unitary_measurements_and_gaussian_noise.for_denoising_paper}, let $S$ be a deterministic matrix and suppose $SDF$ satisfy the \hyperref[loc:rip.statement]{RIP} on
 $\mathcal{T}$. Without loss of generality, assume that the rows of $S$ are ordered such that $S \boldsymbol{d}$ is non-increasing.
Then for $t>0$, the following holds with probability at least
$1-2\exp(-t^2)$. 
\begin{align*}
\lVert \hat{\boldsymbol{x}}- \boldsymbol{x}_0\rVert_2 \leq 9\frac{ \sigma}{\sqrt{ m }} \|\widetilde{D}\trunc(SD \boldsymbol{\alpha})\|_2 \left( \sqrt{ \ell } + \sqrt{\log M} + t\right)&\\
+\lVert \boldsymbol{x}^\perp\rVert + 6\lVert SDF\boldsymbol{x}^\perp\rVert_{2}  + \frac{3}{2}\sqrt{ \varepsilon }.
\end{align*}
\end{theorem}
We defer~\hyperlink{loc:signal_recovery_with_subsampled_unitary_matrix_with_gaussian_noise_on_union_of_subspaces.proof}{the proof}
to \Cref{loc:body.proofs.signal_recovery_with_denoising}.
Notice that the noise error term in~\Cref{loc:signal_recovery_with_subsampled_unitary_matrix_with_gaussian_noise_on_union_of_subspaces.statement} exhibits a denoising effect through its dependence on the factor of $\frac{1}{\sqrt{m}}$.

The factor $\|\widetilde{D}\trunc(SD \boldsymbol{\alpha})\|_2$ admits a few simple upper bounds shown now, and a probabilistic bound with optimized sampling in  \Cref{loc:tail_on_the_noise_sensitivity_for_adapted_with:replacement_sampling.statement}.
\begin{proposition}[Bounds on the noise error]
\label{loc:simple_bound_on_the_gaussian_noise_error_factor.statement}
With any $\boldsymbol{d} \in \mathbb{R}^n_{++}$, $\boldsymbol{\alpha} \in \mathbb{R}^n_{++}$, and $S$ a fixed $m  \times  n$ sampling matrix, we have that
\begin{equation}
\label{eq:first:trunc}
\|\widetilde{D}\trunc(SD \boldsymbol{\alpha})\|_2 \leq  \max(S \boldsymbol{d}) \leq  \max(\boldsymbol{d}).
\end{equation}
Furthermore, with
$I = |\supp\trunc(SD \boldsymbol{\alpha})|$,
\begin{equation}
\label{eq:second:trunc}
\|\widetilde{D}\trunc(SD \boldsymbol{\alpha})\|_2  \le \|(SD^2 \boldsymbol{\alpha})|_{[I]}\|_2.
\end{equation}
\end{proposition}
\begin{proof}[\hypertarget{loc:simple_bound_on_the_gaussian_noise_error_factor.proof}Proof of \Cref{loc:simple_bound_on_the_gaussian_noise_error_factor.statement}]

We write
\begin{displaymath}
\|\widetilde{D}\trunc(SD \boldsymbol{\alpha})\|_2  = \sqrt{\sum_{i = 1}^I (S \boldsymbol{d})_i^2 \trunc(SD \boldsymbol{\alpha})_i^2}.
\end{displaymath}
Under the square root, we find a convex combination of the squared entries of $S \boldsymbol{d}$ with convex coefficients $\trunc(SD \boldsymbol{\alpha})^{.2}$. The first bound in \Cref{eq:first:trunc} follows from bounding the convex combination by the size of the maximal element. To establish \Cref{eq:second:trunc}, it suffices to check that $(SD \boldsymbol{\alpha})|_{[I]}$ dominates the first $I$ entries of $\trunc(SD \boldsymbol{\alpha})$.
\end{proof}
\subsection{Variable density sampling}
\label{loc:body.denoising_for_subsampled_unitary_matrices.variable_density_sampling}
We sample measurements with replacement according to some probability density
vector $\boldsymbol{p}$, as described in~\Cref{loc:sampling_matrix_with_replacement.statement}.
We quantify how well suited a given sampling probability vector $\boldsymbol{p}$
is to a prior set with the following function.
\begin{definition}[Sampling complexity function]
\label{loc:with:replacement_complexity.statement}
For any $\boldsymbol{\alpha} \in \mathbb{R}^n_{++}$ and $\boldsymbol{p} \in (0, 1]^n \cap \Delta^{n-1}$, we define
\begin{displaymath}
\mu(\boldsymbol{\alpha}, \boldsymbol{p}) := \max_{j \in [n]} \frac{\alpha_j}{\sqrt{ p_{j} }}.
\end{displaymath}
\end{definition}
This function appears as a factor in the sample complexity of~\Cref{loc:cs_with_replacement_and_with_denoising_on_unions_of_subspaces.statement}.
\begin{lemma}[Restricted isometry property on unions of subspaces]
\label{loc:rip_of_with:replacement_non:uniform_sampling_matrix_on_a_union_of_subspaces.statement}
Let $\mathcal{T} \subseteq \measfield^n$ be a union of at most $M$ subspaces,
each with dimension bounded by $\ell$ as in \Cref{loc:setup_of_signal_recovery_with_subsampled_unitary_measurements_and_gaussian_noise.for_denoising_paper}. Suppose $F$, $\boldsymbol{\alpha}$, and $\boldsymbol{p}$ are also as in \Cref{loc:setup_of_signal_recovery_with_subsampled_unitary_measurements_and_gaussian_noise.for_denoising_paper}, 
and $D = \Diag(\boldsymbol{d})$ where $d_i = (n p_i)^{- 1 / 2}$. 
Let $\mu$ denote the complexity function from~\Cref{loc:with:replacement_complexity.statement}. For any $t > 0$, if the number of measurements $m$ satisfies
\begin{displaymath}
m \gtrsim \mu^2(\boldsymbol{\alpha}, \boldsymbol{p})(\log \ell + \log M + t^2),
\end{displaymath}
then with probability at least $1-2\exp(-t^2)$, the matrix $SDF$ satisfies the RIP on $\mathcal{T}$, where $S$ is the $m \times n$ with-replacement sampling matrix associated with the probability vector $\boldsymbol{p}$.
\end{lemma}
The proof to \Cref{loc:rip_of_with:replacement_non:uniform_sampling_matrix_on_a_union_of_subspaces.statement} can be found in the appendix, in \Cref{loc:body.proofs.rip_of_non:uniformly_subsampled_matrices}. There, we carefully treat the case of potentially complex matrices acting on real subspaces. The rest of the proof is standard, making use of the matrix Bernstein inequality.

Combining \Cref{loc:rip_of_with:replacement_non:uniform_sampling_matrix_on_a_union_of_subspaces.statement} with \Cref{loc:signal_recovery_with_subsampled_unitary_matrix_with_gaussian_noise_on_union_of_subspaces.statement} yields a compressed sensing result
for variable density sampling.
\begin{theorem}[Compressed sensing with variable density sampling]
\label{loc:cs_with_replacement_and_with_denoising_on_unions_of_subspaces.statement}
Under~\Cref{loc:setup_of_signal_recovery_with_subsampled_unitary_measurements_and_gaussian_noise.for_denoising_paper},
let $\mu$ be the complexity function defined in~\Cref{loc:with:replacement_complexity.statement},
and let $\delta > 0$. Suppose that
\begin{displaymath}
m \gtrsim \mu^2(\boldsymbol{\alpha}, \boldsymbol{p}) \left( \log \ell + \log M + \log \frac{1}{\delta} \right).
\end{displaymath}
Generate $S$, the $m \times n$ with-replacement sampling matrix associated with $\boldsymbol{p}$, and reorder its rows so that $S \boldsymbol{d}$ is with non-increasing entries. Let $D = \Diag(\boldsymbol{d})$ and $\widetilde{D} = \Diag(S \boldsymbol{d})$.
Then, with probability at least $1-\delta$, the following holds.

For any $\boldsymbol{x}_0  \in \field^n$, with $\varepsilon, \hat{\boldsymbol{x}}, \boldsymbol{x}^{\perp}$ as in~\Cref{loc:setup_of_signal_recovery_with_subsampled_unitary_measurements_and_gaussian_noise.for_denoising_paper}, we have that
\begin{align*}
\lVert \hat{\boldsymbol{x}}- \boldsymbol{x}_0\rVert_2 \leq 9\frac{ \sigma}{\sqrt{ m }}  \|\widetilde{D}\trunc(SD \boldsymbol{\alpha})\|_2  \left( \sqrt{ \ell } + \sqrt{2 \log M} + \sqrt{\log \frac{4}{\delta}}\right)&\\
+\lVert \boldsymbol{x}^\perp\rVert + 6\lVert SDF\boldsymbol{x}^\perp\rVert_{2}  + \frac{3}{2}\sqrt{ \varepsilon }.
\end{align*}
\end{theorem}
The proof combines~\Cref{loc:rip_of_with:replacement_non:uniform_sampling_matrix_on_a_union_of_subspaces.statement} and~\Cref{loc:signal_recovery_with_subsampled_unitary_matrix_with_gaussian_noise_on_union_of_subspaces.statement} with a union bound; see \hyperlink{loc:cs_with_replacement_and_with_denoising_on_unions_of_subspaces.proof}{the proof} in~\Cref{loc:appendix.proofs_by_union_bounds}.
\subsection{Optimized sampling}
\label{loc:body.denoising_for_subsampled_unitary_matrices.optimized_sampling}
The complexity function of \Cref{loc:with:replacement_complexity.statement}
specifies how the sample complexity in 
\Cref{loc:rip_of_with:replacement_non:uniform_sampling_matrix_on_a_union_of_subspaces.statement}
depends jointly on the probability vector and the local coherences. For any
fixed vector of local coherences, it is then natural to ask what is the
 probability vector which minimizes the complexity function.
\begin{lemma}[Optimizing the sampling probabilities]
\label{loc:optimize_the_with:replacement_sampling_probabilities.statement}
For a fixed vector $\boldsymbol{\alpha} \in \mathbb{R}_{++}$ and with the function $\mu$ as in~\Cref{loc:with:replacement_complexity.statement}, let $\boldsymbol{p}'$ be the optimized probability vector as in~\Cref{loc:optimized_sampling_vector.statement}. Then
\begin{displaymath}
\min_{\boldsymbol{p} \in (0,1)^n \cap \Delta^{n-1}}\mu(\boldsymbol{\alpha}, \boldsymbol{p})  = \mu(\boldsymbol{\alpha}, \boldsymbol{p}')  =  \|\boldsymbol{\alpha}\|_2,
\end{displaymath}
where $\boldsymbol{p}'$ is the unique minimizer.
\end{lemma}
We omit the proof, as it is straightforward to show that any variation away from
$\boldsymbol{p}'$ increases the objective function.

When the sampling scheme is adapted to the prior set as in \Cref{loc:optimize_the_with:replacement_sampling_probabilities.statement}, the noise factor $\|\widetilde{D}\trunc(SD \boldsymbol{\alpha})\|_2$ admits a simple upper bound.
\begin{lemma}[Bound on the noise error for optimized sampling]
\label{loc:tail_on_the_noise_sensitivity_for_adapted_with:replacement_sampling.statement}
For any fixed local coherence vector $\boldsymbol{\alpha} \in \mathbb{R}^n_{++}$, let $S$ be a $m \times n$ with-replacement sampling matrix associated with the optimized probability vector $\boldsymbol{p}'$. Let $D = \Diag(\boldsymbol{d})$ where $d_i = (n p'_i)^{- 1 / 2}$. Then for any $t>0$
\begin{displaymath}
\|\widetilde{D} \trunc(SD \boldsymbol{\alpha})\|_2  \le \|\boldsymbol{\alpha}\|_2
\min\left( \frac{1}{\sqrt{t}},  \frac{1}{\sqrt{n} \min(\boldsymbol{\alpha})}\right)
\end{displaymath}
with probability at least $1-t$.
\end{lemma}
As can be seen in \hyperlink{loc:tail_on_the_noise_sensitivity_for_adapted_with:replacement_sampling.proof}{the proof} of \Cref{loc:tail_on_the_noise_sensitivity_for_adapted_with:replacement_sampling.statement} in \Cref{loc:body.proofs.signal_recovery_with_denoising}, this upper bound is made possible by
cancellations which are unique to the optimized
sampling scheme, which came as a surprise to the authors.

We now have all the necessary results to prove our main result,
\Cref{loc:optimized_with:replacement_cs_on_union_of_subspaces.statement}.
Indeed,
\Cref{loc:optimized_with:replacement_cs_on_union_of_subspaces.statement}
holds when \Cref{loc:rip_of_with:replacement_non:uniform_sampling_matrix_on_a_union_of_subspaces.statement},
\Cref{loc:signal_recovery_with_subsampled_unitary_matrix_with_gaussian_noise_on_union_of_subspaces.statement},
and \Cref{loc:tail_on_the_noise_sensitivity_for_adapted_with:replacement_sampling.statement} hold simultaneously; see
\hyperlink{loc:optimized_with:replacement_cs_on_union_of_subspaces.proof}{the proof} in
\Cref{loc:appendix.proofs_by_union_bounds} for the details of the appropriate union bound.
\subsection{Applications}
\label{loc:body.denoising_for_subsampled_unitary_matrices.applications}
The assumption that the prior set $\mathcal{Q}$ is contained in a union of low-dimensional
subspaces is effective for two important examples: generative priors and sparse priors.
We begin with the generative case.

\textbf{Generative priors}
\label{loc:body.denoising_for_subsampled_unitary_matrices.applications.generative_priors}
\begin{definition}[(k,d,n)-Generative neural network~{\cite[Definition 1.1]{berkCoherenceParameterCharacterizing2022}}]
\label{loc:(kdn):generative_network.verbatim_from_coherence_paper}
  Fix the integers $2 \leq k := k_{0} \leq k_{1}, \ldots, k_{d}$ where
  $k_d := n < \infty$, and suppose for $i \in [d]$ that
  $W^{(i)} \in \reals^{k_{i}\times k_{i-1}}$. A $(k,d,n)$-generative network is
  a function $G : \reals^{k} \to \reals^{n}$ of the form
\begin{displaymath}
G(z) := W^{(d)} \sigma \left( \cdots W^{(2)} \sigma \left( W^{(1)} z \right) \right).
\end{displaymath}
\end{definition}
In \cite[Remark S2.2]{berkModeladaptedFourierSampling2023}, it is
stated that a (k,d,n)-neural network $G:\mathbb{R}^k  \to  \mathbb{R}^n$
has its range contained in $N$ subspaces each of dimension no more than
$k$, with
\begin{displaymath}
\log N  \le k \sum_{i = 1}^{d-1} \log\left( \frac{2ek_i}{k}\right).
\end{displaymath}
Then we let $\mathcal{T} = \Delta(\range(G)-\range(G))$, where $\Delta$ is the
following set expansion operator.
\begin{definition}[Piecewise linear expansion~{\cite{berkCoherenceParameterCharacterizing2022}[Definition 2.1]}]
\label{loc:piecewise_linear_expansion.statement_for_cones}
Let $\mathcal{C} \subseteq \reals^{n}$ be a union of $N$ convex cones:
$\mathcal{C} = \bigcup\limits_{i=1}^N \mathcal{C}_i$. Define the piecewise linear expansion to be
\begin{align*}
\Delta(\mathcal{C}) := \bigcup\limits_{i=1}^N \Span(\mathcal{C}_i) = \bigcup_{i=1}^N
(\mathcal{C}_i - \mathcal{C}_i).
\end{align*}
\end{definition}
By the properties of this set operator~\cite[Remark S3.1]{berkCoherenceParameterCharacterizing2022}, $\mathcal{T}$ is contained in no more than $N^2$ subspaces each
of dimension no more than $2k$. So with this $\mathcal{T}$, \Cref{loc:optimized_with:replacement_cs_on_union_of_subspaces.statement}
and \Cref{loc:cs_with_replacement_and_with_denoising_on_unions_of_subspaces.statement} apply to the generative
setting with $M = N^2$ and $\ell = 2k$. We thus improve upon previous works in generative compressed sensing
by providing denoising error bounds. We do caution, however, that our results applied to
the generative setting are not state-of-the-art in all
respects: results which depend on a
tighter notion of local coherences can be found in~\cite{adcockUnifiedFrameworkLearning2023}.

\textbf{Sparse priors}

\label{loc:body.denoising_for_subsampled_unitary_matrices.applications.sparse_priors}
For a sparse prior set $\mathcal{Q}$, $\mathcal{T} = \mathcal{Q}-\mathcal{Q}$
is the set of $2k-$sparse vectors. It is a union of $M$ subspaces each of
dimension no more than $2k$, where $M$ admits the bound
\begin{displaymath}
\log M = \log \binom{n}{2k}  \le 2k \log \left( \frac{n}{2k}\right).
\end{displaymath}
While specialized treatments of sparse priors achieve tighter sample complexity bounds, our results make a significant advance over prior works by providing the first theoretical guarantees of denoising behavior in the sparse setting.
\section{Numerics}
\label{loc:body.numerics}
\subsection{Experiments with generative priors}
In this section, we conduct experiments with generative priors as introduced in \Cref{loc:(kdn):generative_network.verbatim_from_coherence_paper}.
\subsubsection*{Model and Dataset}
We train a generative model on images of the CelebFaces Attributes Dataset (CelebA). To facilitate training, we further center and crop the color images to 256 by 256, leading to 256 × 256 × 3 = 196608 pixels per image. On this dataset, we train a Realness GAN ~\cite{xiangli2020realrealquestion} with the same training setup, except that we replace the final Tanh layer with HardTanh, a linearized version of Tanh. For more details on the model training, please refer to~\cite{xiangli2020realrealquestion}.  The unitary transform, $F$, that we consider is a channel-wise concatenation of the 2D Fourier Transform.

\subsubsection*{Coherence}
As noted in~\cite{berkModeladaptedFourierSampling2023}, calculating local coherence for generative models appears computationally intractable. However, there is a straightforward way to make an effective estimate for this quantity by using a finite subset of the prior set. As in~\cite{adcockCAS4DLChristoffelAdaptive2022, berkCoherenceParameterCharacterizing2022}, we sample 5000 latent codes from a standard normal distribution and input these codes to the generative model to create a batch of images.  To compute the coherence of a certain measurement vector, we take all differences of images in this batch, normalize them, compute the absolute value of the inner product of each of them with the measurement vector, and take the largest of these values.

\subsubsection*{Signal Reconstruction}
We generate a with-replacement sampling matrix $\boldsymbol{S}$ according to the optimized probability vector $\boldsymbol{p}'$ with $p^*_i = \frac{\alpha_i^2}{\|\alpha_i\|_2}$. We generate a synthetic signal $\boldsymbol{x_0} = G(\boldsymbol{z})$ where $\boldsymbol{z}$ is a latent code taken from the standard normal distribution and $G$ is the trained generative neural net described above. We then set $\boldsymbol{b} = S F \boldsymbol{x_0} + \boldsymbol{\eta}$ where the noise vector $\boldsymbol{\eta}$ has Gaussian entries with variance $\sigma^2/m$ for various noise levels $\sigma$. We compare the performance of the signal recovery under optimized sampling to uniform sampling, which takes the probability vector $p_i =\frac{1}{n}.$

In the generative setting, the optimization problem of~\Cref{eq:opt:reconstruct} in~\Cref{loc:setup_of_signal_recovery_with_subsampled_unitary_measurements_and_gaussian_noise.for_denoising_paper} becomes 
\begin{displaymath}
\minimize_{\boldsymbol{z} \in \mathbb{R}^k} \, \left\lVert \widetilde{D} SF\boldsymbol{G(z)}-\widetilde{D}\boldsymbol{b} \right\rVert_{2}^2.
\end{displaymath}
To optimize, we use $AdamW$, with $\beta_1 = 0.99, \beta_2=0.999$, and $lr=0.0003$ for 20000 iterations, yielding an optimized latent vector $\hat{\boldsymbol{z}}$.  We validate the performance of different sampling schemes using relative recovery error (rre), where $\text{rre}(\boldsymbol{x}_0,G(\hat{\boldsymbol{z}})) = \frac{\|\boldsymbol{x}_0-G(\hat{\boldsymbol{z}})\|_2}{\|x_0\|_2}$.
We repeat this experiment 256 times for each noise level to report an average relative recovery error, as shown in~\Cref{fig:es_vs_wr}. This exhibits the significant improvement in sample complexity gained by optimizing the sampling distribution. In~\Cref{fig:denoise_main}, we estimate the slope of the recovery error in log scale with a least-squares fit in log scale, thereby estimating the dependence on $m$. We compare the values of the slopes to the rate of $1/\sqrt{m}$ in our denoising bounds. In all figures, we display the geometric mean of the data and the geometric standard error as uncertainty (which is the statistical uncertainty of the geometric mean estimator).
\begin{figure}[h]
    \centering
\includegraphics[width=\linewidth]{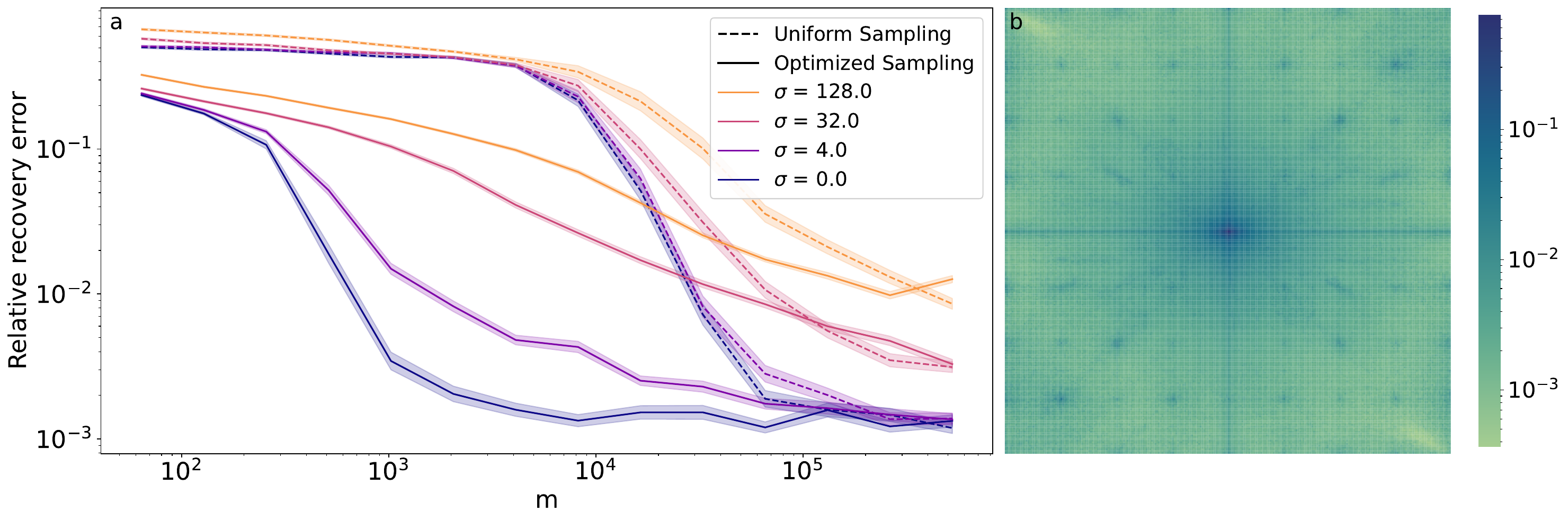} 
    \caption{a) Relative error for optimized sampling versus uniform sampling with generative models.
    b)  the red channel of the local coherence. } 
    \label{fig:es_vs_wr}
\end{figure}


\begin{figure}[h]
    \centering
\includegraphics[width=0.60\linewidth]{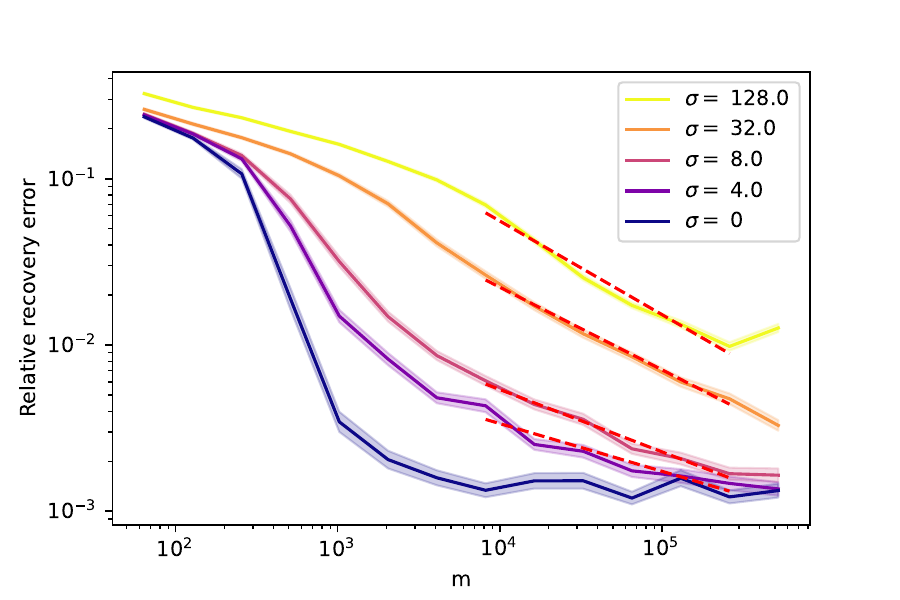} 
    \caption{Relative recovery error for 256 repeated experiments at each $m$ and $\sigma$. We perform least-squares fits in log space for $m \in [10^3.9, 10^5.4]$ (roughly after the phase transition, but before saturation).   The slopes of the fitted lines from top to bottom are $-0.56,-0.50,-0.37,-0.29 $ respectively.} 
    \label{fig:denoise_main}
\end{figure}



\subsubsection*{Visual evidence}
In \Cref{fig: faces}, we contrast uniform sampling versus optimized sampling on real faces. We add noise directly to the signal before taking the measurements, thus using a slight variation on our noise model. The first column contains the original faces and the second column contains the ones with high noise added, $\sigma= 128$. The odd rows show the faces recovered by optimized sampling and the even rows show the faces recovered by uniform sampling. The percentage at the top of the columns is the proportion of measurements taken divided by ambient dimension.  Visually, it appears that optimized sampling recovers faces well from about 0.5-2\% percent of measurements, whereas uniform sampling recovers well from about 8\%.

\begin{figure} \label{fig: faces}
\centering
\includegraphics[width=0.70\textwidth]{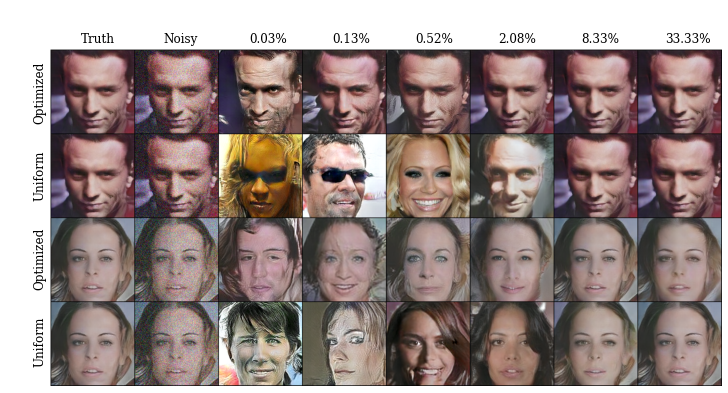}
\caption{Noisy faces recovered by optimized sampling and uniform sampling. }
\end{figure}

\subsection{Experiments with Sparse priors}
\label{loc:body.numerics.experiments_with_sparse_priors}
In~\Cref{fig:sparse_recov.pdf} we recover a signal which is sparse in the Haar wavelet basis of level $5$ from
subsampled Fourier measurements. The signal 
is obtained by resizing the cameraman image to be $256$ by $256$,
and then truncating it to its $60$
largest coefficients in the sparsity basis.

We subsample the measurements according to the optimized sampling distribution,
which we obtain from the local coherences of the Fourier measurement vectors
with respect to the vectors of the Haar basis. Note that this is not quite the
local coherences that are found in our theory, but are nonetheless a good
estimate. The local coherence with respect to the sparsity basis is
a notion of local coherence that appears in the literature for optimized
sampling on sparse signals~\cite{krahmerStableRobustSampling2014}.

To solve the optimization problem in \Cref{eq:opt:reconstruct}, we employ an algorithm that has two steps: first we run basis pursuit denoising
using SPGL1~\cite{vandenbergProbingParetoFrontier2009}. From the solution, we estimate the support to correspond to the $60$ largest wavelet coefficients.
Then, we find the best fit to the measurements on this support by solving a least-squares problem.

In \Cref{fig:sparse_recov.pdf} we follow the same procedure as \Cref{fig:denoise_main}.
\begin{figure}[h]
\centering
\includegraphics[width=0.60\textwidth]{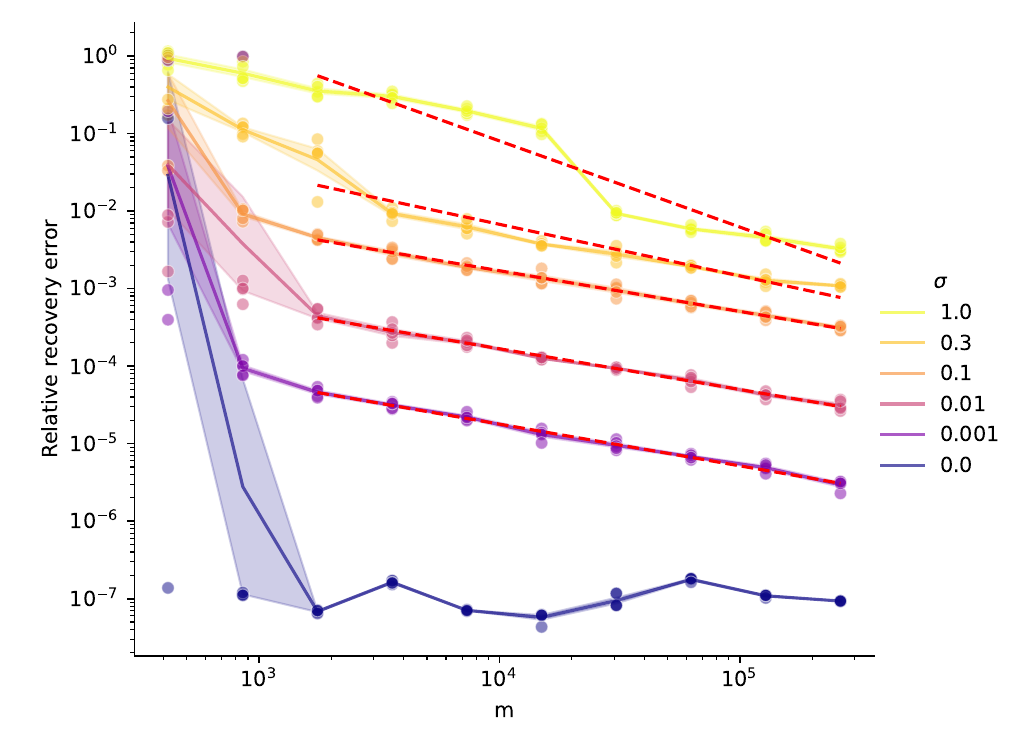}
\caption{Relative errors of five signal recovery experiments for each value of $m$ and $\sigma$. We perform least-squares fits in log space for $m>300$ (roughly after the phase transition). The fits have slopes of $-1.11$, $-0.67$, $-0.53$, $-0.53$, $-0.54$ respectively, in order of decreasing $\sigma$. The ambient dimension is $n=6.6  \cdot 10^4$, we run experiments up to $m \approx 4n$.\label{fig:sparse_recov.pdf}}
\end{figure}
\section{Noise robustness in the literature}
\label{loc:body.noise_robustness_in_the_literature}
As discussed in the introduction, we believe the noise model used in our work is more principled and realistic than what appears in prior literature. 
To discuss the different treatments of noise in earlier works, let us first consider~\cite[Theorem 2.1]{berkModeladaptedFourierSampling2023}, a compressed sensing result with an optimized sampling scheme and deterministic noise on a generative prior. 
\begin{proposition}[Generative CS with deterministic noise]
\label{loc:generative_signal_recovery_with_optimized_with:replacement_scheme_and_deterministic_noise.even_shorter}
Consider the prior set $\mathcal{Q}:= \range(G)  \subseteq \mathbb{R}^n$
for $G:\mathbb{R}^k  \to  \mathbb{R}^n$ a $(k,d,n)$-generative network as in~\Cref{loc:(kdn):generative_network.verbatim_from_coherence_paper}.
Let $F \in \mathbb{C}^{n  \times n}$ be a unitary matrix with local
coherences $\boldsymbol{\alpha} \in \mathbb{R}_{++}$ with respect to
$\mathcal{T} := \Delta(\mathcal{Q}-\mathcal{Q})  \subseteq \mathbb{C}^n$.
Fixing $\delta>0$, suppose that
\begin{displaymath}
m  \gtrsim \|\boldsymbol{\alpha}\|_2^2 \left( kd \log\left( \frac{n}{k} \right)+ \log \frac{2}{\delta}\right),
\end{displaymath}
that $S$ is an $m \times n$ \hyperref[loc:sampling_matrix_with_replacement.statement]{with-replacement sampling matrix} governed by the optimized probability vector $\boldsymbol{p}'$, and that $D := \Diag\left(\boldsymbol{d}\right)$ with $d_i = (n p_i')^{- 1 / 2}$.
Then the following holds with probability at least $1-\delta$.

For any $\boldsymbol{x}_0 \in \mathbb{R}^n$ with $\boldsymbol{x}^\perp = \boldsymbol{x}_{0} - \proj_{\range(G)} \boldsymbol{x}_0$, let $\boldsymbol{b} = SF\boldsymbol{x}_{0} + \boldsymbol{\eta}$,
where $\boldsymbol{\eta}  \in \mathbb{C}$. Then for $\hat{\boldsymbol{z}} \in \mathbb{R}^k$ satisfying
\begin{displaymath}
\left\lVert  SDFG(\hat{\boldsymbol{z}})-\widetilde{D}\boldsymbol{b} \right\rVert_{2}^2 \leq  \min_{\boldsymbol{x} \in \mathcal{Q}}\left\lVert  SDF\boldsymbol{x} - \widetilde{D}\boldsymbol{b} \right\rVert_2^2+\varepsilon
\end{displaymath}
for some $\varepsilon>0$, it holds that
\begin{align*}
\|G(\hat{\boldsymbol{z}}) - \boldsymbol{x}_0\|_2 %
  \leq \|\boldsymbol{x}^\perp\|_2 + 3\|SDF\boldsymbol{x}^\perp\|_2 + 3 \|\widetilde{D}\boldsymbol{\eta}\|_2 + \frac{3}{2}\varepsilon.
\end{align*}
\end{proposition}
Notice the error term $\|\widetilde{D}\boldsymbol{\eta}\|_2$, and recall that $\widetilde{D} = \Diag(S\boldsymbol{d})$ for a fixed preconditioning vector $\boldsymbol{d} \in \mathbb{R}^n$. Each entry $d_i$ of $\boldsymbol{d}$ is associated through its index to $f_i$, a specific measurement vector. Therefore, this term, $\|\widetilde{D}\boldsymbol{\eta}\|_2$, scales the noise by coefficients which depend on a random sample of measurement vectors.

This type of theoretical result is dubious in application. A practitioner hoping to
use such a result would have to first sample the CS matrix,
and only know of a high probability bound describing its quality after
having sampled the CS matrix (and would not be allowed to re-sample it if the theoretical
guarantees are to hold).

Though this bound has its shortcomings, we believe that it is an
accurate description of the robustness to deterministic noise of subsampled
unitary measurements with non-uniform local coherences. Since different
measurement vectors have differing degrees of alignment (i.e., differing local
coherences) with the prior, the dynamic range of non-noisy measurements vary
accordingly, resulting in effectively large (in magnitude) measurements for rows
that are coherent with the prior and smaller measurements for rows that are
incoherent. This in turn leads to typically higher Signal-to-Noise ratios (SNRs) for
``coherent" measurements and lower SNRs for incoherent ones. It would be
possible to obtain a signal recovery bound which does not depend on the measurement
matrix by using a uniform lower-bound on all local coherences, but such a
bound would under-estimate of the true robustness to noise by a potentially wide margin.
Indeed, a measurement vector that is almost orthogonal to the prior set
would almost never be sampled, and yet would affect such a bound significantly.

Sampling-dependent noise sensitivity appears in earlier works in subtle ways for
both bounded and deterministic noise, even when the recovery error bound
does not feature a noise error term that is sampling-dependent. In the
two examples we discuss below, this type of noise sensitivity instead
appears implicitly in the signal acquisition model.

Krahmer and Ward~\cite{krahmerStableRobustSampling2014} consider \emph{bounded} (adversarial)
noise contained in an ellipsoid. Notably, the dimensions of this ellipsoid
scale inversely with the entry-wise preconditioning associated with the measurement
vectors, making the ellipsoid's shape dependent on the specific selection of measurement vectors
in the CS matrix. The worst-case
noise then scales element-wise with the local coherences of the sampled measurement
vectors in such a way that the impact of varying local coherences 
on the noise sensitivity is cancelled out. One downside of such an approach
is that the constraint on the adversarial noise is dependent on the 
random CS matrix.


Adcock, Cardenas, and Dexter~\cite{adcockUnifiedFrameworkLearning2023} instead consider deterministic
noise, with a different model for signal acquisition than that of~\Cref{loc:generative_signal_recovery_with_optimized_with:replacement_scheme_and_deterministic_noise.even_shorter}. 
Their CS matrix has rows with varying norms, which modifies the strength of the
signal in the measurements relative to the noise. From the perspective
of the present paper, they take the measurement
matrix to be $SDF$, the preconditioned CS matrix at the time of measurement, whereas we
consider $SF$, and only introduce the preconditioning in the optimization step.
By preconditioning the CS matrix it is possible to correct for discrepancies
in dynamic range that arise from differing local coherences. 

We believe that our signal acquisition model is more realistic than the
alternatives discussed so far.
Measurement vectors model measurement devices, and there is no reason
to believe that measurement devices would either downscale the noise
element-wise or strengthen the signal element-wise in a way that depends
on the choice of measurements. That being said, results featuring
differing measurement acquisition models are often mathematically equivalent,
meaning that we can convert between two measurement acquisition models
by simultaneously adjusting the error terms in the recovery bound with
a variable substitution. For
example, we do this in \hyperlink{loc:noise_corollary_from_deterministic_results_on_union_of_subspaces.proof}{the proof} of \Cref{loc:noise_corollary_from_deterministic_results_on_union_of_subspaces.statement}.

In this work, we find simpler and meaningful bounds by
considering Gaussian noise. This also means that we provide, to our knowledge,
the first robust bounds for variable density sampling in compressed sensing
which do not suffer from the aforementioned limitations.

It may be argued that signal recovery bounds in the presence of Gaussian noise can be obtained as
a straightforward corollary of previous work on deterministic
noise. But as we demonstrate in the appendix, such results will not exhibit a
denoising behavior. In \Cref{loc:appendix.gaussian_noise_corollary_from_deterministic_noise}, we derive such a result as a corollary of \cite[Theorem 3.6]{adcockUnifiedFrameworkLearning2023},
and show that it falls short relative to our specialized treatment of stochastic noise.
\section{Proofs}
\label{loc:body.proofs}
\subsection{Signal recovery with denoising}
\label{loc:body.proofs.signal_recovery_with_denoising}
\begin{proof}[\hypertarget{loc:tail_on_the_noise_sensitivity_for_adapted_with:replacement_sampling.proof}Proof of \Cref{loc:tail_on_the_noise_sensitivity_for_adapted_with:replacement_sampling.statement}]

Recall that the diagonal entries of $D$ are $d_i = \frac{1}{\sqrt{n p'_i}}  = \frac{\|\boldsymbol{\alpha}\|_2}{\sqrt{n} \alpha_i}$. A first bound is found in \Cref{loc:simple_bound_on_the_gaussian_noise_error_factor.statement}:
\begin{displaymath}
\|\widetilde{D}\trunc(SD \boldsymbol{\alpha})\|_2 \leq d_{1} = \frac{\lVert \boldsymbol{\alpha}\rVert_{2}}{\sqrt{ n } \min(\boldsymbol{\alpha})}.
\end{displaymath}
We find a second bound to combine with the first. With the fact that
$D \boldsymbol{\alpha} = \frac{\|\boldsymbol{\alpha}\|_2}{\sqrt{n}} \allone$
for $\allone$ the vector with all entries $1$,
the second bound in \Cref{loc:simple_bound_on_the_gaussian_noise_error_factor.statement}
becomes
\begin{align*}
\|\widetilde{D}\trunc(SD \boldsymbol{\alpha})\|_2  
&\le \|SD^2 \boldsymbol{\alpha}\|_2\\
& \le \frac{\|\boldsymbol{\alpha}\|_2}{\sqrt{n}} \|SD \allone\|_2.
\end{align*}
With the fact that $SD \allone =  S \boldsymbol{d}$,
\begin{displaymath}
\mathbb{E} \|S \boldsymbol{d}\|_2^2 = \sum_{i = 1}^m \sum_{j = 1}^n p'_j \frac{n}{m} \frac{1}{n p'_j} = n.
\end{displaymath}
Using Markov's inequality, we find that with probability at least $1-\delta$,
\begin{displaymath}
\|\widetilde{D}\trunc(SD \boldsymbol{\alpha})\|_2  \le \frac{\|\boldsymbol{\alpha}\|_2}{\sqrt{\delta}}.
\end{displaymath}
Combining the two bounds yields the result.
\end{proof}
\begin{proof}[\hypertarget{loc:signal_recovery_with_subsampled_unitary_matrix_with_gaussian_noise_on_union_of_subspaces.proof}Proof of \Cref{loc:signal_recovery_with_subsampled_unitary_matrix_with_gaussian_noise_on_union_of_subspaces.statement}]

Let $\bar{\boldsymbol{x}} = \proj_{\mathcal{Q}}\boldsymbol{x}_0$ and $\boldsymbol{h}= \hat{\boldsymbol{x}}-\boldsymbol{x}_0$. By definition of $\hat{\boldsymbol{x}}$,
\begin{align*}
\left\lVert  SDF\hat{\boldsymbol{x}}-\widetilde{D}\boldsymbol{b} \right\rVert_{2}^2 
&\leq  \min_{\boldsymbol{x}\in  \mathcal{Q}}\lVert  SDF\boldsymbol{x} - \widetilde{D}\boldsymbol{b} \rVert_2^2+\varepsilon\\
&\leq \lVert SDF\bar{\boldsymbol{x}} - \widetilde{D} \boldsymbol{b}\rVert_2^2 + \varepsilon\\
& = \lVert SDF (\bar{\boldsymbol{x}}- \boldsymbol{x}_0) + SDF\boldsymbol{x}_0-\widetilde{D}\boldsymbol{b}\rVert_2^2 + \varepsilon\\
&\leq  \lVert SDF\boldsymbol{x}^\perp - \widetilde{D}\boldsymbol{\eta}\rVert_2^2 + \varepsilon\\
&\leq \lVert SDF\boldsymbol{x}^\perp\rVert_2^2+ 2 \mathcal{R}\langle SDF\boldsymbol{x}^\perp, \widetilde{D}\boldsymbol{\eta}\rangle + \lVert \widetilde{D} \boldsymbol{\eta}\rVert_2^2+\varepsilon.
\end{align*}
Then consider the l.h.s. of the above inequality:
\begin{align*}
\lVert  SDF\hat{\boldsymbol{x}}-\widetilde{D}\boldsymbol{b} \rVert_{2}^2
&= \lVert  SDF\hat{\boldsymbol{x}}-(SDF\boldsymbol{x}_0+\widetilde{D}\boldsymbol{\eta} ) \rVert_{2}^2\\
&= \lVert SDF\boldsymbol{h}\rVert_2^2 - 2 \mathcal{R}\langle SDF\boldsymbol{h}, \widetilde{D}\boldsymbol{\eta}\rangle + \lVert \widetilde{D}\boldsymbol{\eta}\rVert_{2}^2.
\end{align*}
Combining these equations, we get
\begin{displaymath}
\lVert SDF\boldsymbol{h}\rVert_{2}^2 \leq \lVert SDF\boldsymbol{x}^\perp\rVert_2^2+ 2 \mathcal{R}\langle SDF\boldsymbol{x}^\perp, \widetilde{D}\boldsymbol{\eta}\rangle + 2 \mathcal{R}\langle SDF\boldsymbol{h}, \widetilde{D}\boldsymbol{\eta}\rangle + \varepsilon.
\end{displaymath}
Substitute $\boldsymbol{h}= \tilde{\boldsymbol{h}}- \boldsymbol{x}^\perp$,
where $\tilde{\boldsymbol{h}} \in \mathcal{Q}-\mathcal{Q}$,
to find that
\begin{align*}
&\lVert SDF\tilde{\boldsymbol{h}} \rVert_{2}^2 - 2 \mathcal{R}\langle SDF \tilde{\boldsymbol{h}}, SDF\boldsymbol{x}^\perp\rangle + \lVert SDF \boldsymbol{x}^\perp\rVert_{2}^2 \\ 
\leq  &\lVert SDF\boldsymbol{x}^\perp\rVert_2^2+ 2 \mathcal{R}\langle
SDF\boldsymbol{x}^\perp, \widetilde{D}\boldsymbol{\eta}\rangle + 2 \mathcal{R}\langle
SDF\tilde{\boldsymbol{h}}, \widetilde{D}\boldsymbol{\eta}\rangle - 2 \mathcal{R}\langle SDF
\boldsymbol{x}^\perp, \widetilde{D}\boldsymbol{\eta}\rangle+\varepsilon,
\end{align*}
which reduces to
\begin{displaymath}
\lVert SDF\tilde{\boldsymbol{h}} \rVert_{2}^2 \leq 2 \mathcal{R}\langle SDF \tilde{\boldsymbol{h}}, SDF\boldsymbol{x}^\perp\rangle + 2 \mathcal{R}\langle SDF\tilde{\boldsymbol{h}}, \widetilde{D}\boldsymbol{\eta}\rangle + \varepsilon.
\end{displaymath}
Then using the \hyperref[loc:rip.statement]{RIP} on the l.h.s., it follows that
\begin{align*}
\left( 1 - \frac{1}{3} \right)^2\lVert \tilde{\boldsymbol{h}} \rVert_{2}^2 &\leq 2 \mathcal{R}\langle SDF\tilde{\boldsymbol{h}}, SDF\boldsymbol{x}^\perp\rangle + 2 \mathcal{R}\langle SDF\tilde{\boldsymbol{h}}, \widetilde{D} \boldsymbol{\eta}\rangle + \varepsilon\\
&\leq 2 \lVert \tilde{\boldsymbol{h}}\rVert_2 \sup_{\boldsymbol{y} \in \mathcal{T} \cap B_2}\mathcal{R}\langle SDF\boldsymbol{y}, SDF\boldsymbol{x}^\perp\rangle + \lVert \tilde{\boldsymbol{h}}\rVert_2 \sup_{\boldsymbol{y} \in \mathcal{T} \cap B_2} 2 \mathcal{R}\langle SDF\boldsymbol{y}, \widetilde{D} \boldsymbol{\eta}\rangle + \varepsilon.
\end{align*}
Let 
\begin{displaymath}
E = \sup_{\boldsymbol{y} \in \mathcal{T} \cap B_2}2 \mathcal{R}\langle SDF\boldsymbol{y}, SDF\boldsymbol{x}^\perp\rangle + \sup_{\boldsymbol{y} \in \mathcal{T} \cap B_2} 2 \mathcal{R}\langle SDF\boldsymbol{y}, \widetilde{D} \boldsymbol{\eta}\rangle.
\end{displaymath}
We find a bound on $\lVert \tilde{\boldsymbol{h}}\rVert_2$ by considering the quadratic formula $ax^2 +bx +c$ with coefficients $a = 4 /9, b= -E, c = -\varepsilon$,
\begin{align*}
\lVert \tilde{\boldsymbol{h}}\rVert_2 &\leq \frac{1}{2 \cdot 4 /9}\left( E + \sqrt{E^2 + 4\varepsilon \cdot 4 /9 } \right) \\
&\leq \frac{1}{2 \cdot 4/9}(2E + 2 \cdot 2 / 3 \sqrt{ \varepsilon })\\
&= \frac{9}{4}E + \frac{3}{2}\sqrt{ \varepsilon }.
\end{align*}
From the triangle inequality
\begin{displaymath}
\lVert \boldsymbol{h}\rVert_2 \leq  \lVert \tilde{\boldsymbol{h}}\rVert_{2} + \lVert \boldsymbol{x}^\perp\rVert_{2},
\end{displaymath}
it follows that
\begin{displaymath}
\lVert \boldsymbol{h}\rVert_2 \leq \frac{9}{4}E + \frac{3}{2}\sqrt{ \varepsilon } + \lVert \boldsymbol{x}^\perp\rVert_{2}.
\end{displaymath}
Then it remains only to bound $E$, which we do with the following lemma.
\begin{lemma}[Bounding the noise error term]
\label{loc:gaussian_noise_error_term_for_unevenly_subsampled_unitary_measurements_on_a_union_of_subspaces.statement}
Under~\Cref{loc:setup_of_signal_recovery_with_subsampled_unitary_measurements_and_gaussian_noise.for_denoising_paper}, let $S$ be a deterministic sampling matrix and suppose that $SDF$ satisfies the \hyperref[loc:rip.statement]{RIP} on
$\mathcal{T}$. Let $t>0$. Then
\begin{displaymath}
\sup_{\boldsymbol{y} \in \mathcal{T} \cap B_2} 2 \mathcal{R}\langle SDF\boldsymbol{y}, \widetilde{D} \boldsymbol{\eta}\rangle \le 4 \frac{\sigma}{\sqrt{ m }}\|\widetilde{D}\trunc(SD \boldsymbol{\alpha})\|_2\left( \sqrt{\ell} + \sqrt{\log M} + t \right)
\end{displaymath}
with probability at least $1-2 \exp(-t^2)$.
\end{lemma}
To finish the proof of \Cref{loc:signal_recovery_with_subsampled_unitary_matrix_with_gaussian_noise_on_union_of_subspaces.statement}, it remains only to bound
\begin{displaymath}
\sup_{\boldsymbol{y} \in \mathcal{T} \cap B_2} 2 \mathcal{R}\langle SDF\boldsymbol{y}, SDF\boldsymbol{x}^\perp\rangle.
\end{displaymath}
Since $SDF$ has the RIP on $\mathcal{T}$,
\begin{displaymath}
\sup_{\boldsymbol{y} \in  \mathcal{T}\cap B_2} \lVert SDF\boldsymbol{y}\rVert_2 \leq  \frac{4}{3}.
\end{displaymath}
With the Cauchy-Schwartz inequality we find that
\begin{displaymath}
\sup_{\boldsymbol{y} \in \mathcal{T} \cap B_2} 2 \mathcal{R}\langle SDF\boldsymbol{y},SDF\boldsymbol{x}^\perp\rangle
 \le \sup_{\boldsymbol{y} \in  \mathcal{T}\cap B_2} \lVert SDF\boldsymbol{y}\rVert_2\lVert SDF\boldsymbol{x}^\perp\rVert_{2}
 \le \frac{8}{3}\lVert SDF\boldsymbol{x}^\perp\rVert_{2}.
\end{displaymath}
Therefore,
\begin{displaymath}
E \leq 4\frac{ \sigma}{\sqrt{ m }} \|\widetilde{D}\trunc(SD \boldsymbol{\alpha})\|_2 \left( \sqrt{ \ell } + \sqrt{\log M} + t\right) +\frac{8}{3}\lVert SDF\boldsymbol{x}^\perp\rVert_{2}
\end{displaymath}
with probability at least $1-2\exp(-t^2)$, which yields the result.
\end{proof}
\begin{proof}[\hypertarget{loc:gaussian_noise_error_term_for_unevenly_subsampled_unitary_measurements_on_a_union_of_subspaces.proof}Proof of \Cref{loc:gaussian_noise_error_term_for_unevenly_subsampled_unitary_measurements_on_a_union_of_subspaces.statement}]

In the case where $\mathbb{K}$ is $\mathbb{C}$, note that in \Cref{loc:gaussian_noise_error_term_for_unevenly_subsampled_unitary_measurements_on_a_union_of_subspaces.statement} we only contend with the real
part of the complex inner product. The space $\mathbb{C}^m$ over the real field and
with the real inner product
is isometric and isomorphic to $\mathbb{R}^{2m}$. We therefore map
vectors that are in
$\mathbb{C}^m$ into $\mathbb{R}^{2m}$ in the canonical way, and at times
discuss the same vectors as being in $\mathbb{C}^m$. In this proof only, we denote by $\langle  \cdot ,  \cdot \rangle$ the canonical inner
product in $\mathbb{R}^{2m}$, and by $\| \cdot \|$ the operator norm
in $\mathbb{R}^{2m}$. The matrices $D  \in \mathbb{R}^{2n  \times  2n}$, $F  \in \mathbb{R}^{2n  \times n}$, and
$S \in \mathbb{R}^{2m  \times 2n}$ are defined so as to preserve the linear map structure.
For simplicity, in the rest of this proof we use $\mathbb{K}^m$ to refer
to $\mathbb{R}^{2m}$ if $\mathbb{K}$ is $\mathbb{C}$, and to $\mathbb{R}^m$ if $\mathbb{K}$ is $\mathbb{R}$.

Let $\mathcal{U} \subseteq \mathcal{T}$ be a subspace of dimension at most $\ell$ in $\mathbb{R}^n$. We first consider the simpler problem of bounding the random variable $\sup_{\boldsymbol{y} \in \mathcal{U} \cap B_2} 2\langle SDF\boldsymbol{y}, \widetilde{D} \boldsymbol{\eta}\rangle$ with high probability.
Let $\bar{\mathcal{U}} = SDF\mathcal{U} \subseteq \measfield^{m}$, and $\bar{\boldsymbol{y}} = SDF \boldsymbol{y} \in \measfield^{m}$. Then
\begin{align*}
\sup_{\boldsymbol{y} \in \mathcal{U} \cap B_2} 2\langle SDF\boldsymbol{y}, \widetilde{D} \boldsymbol{\eta}\rangle &= \frac{\sigma}{\sqrt{ m }}\sup_{\boldsymbol{y}\in \mathcal{U} \cap B_2}2\langle S DF\boldsymbol{y}, \widetilde{D} \boldsymbol{g} \rangle \\
 &\leq  \frac{\sigma}{\sqrt{ m }} \frac{8}{3}\sup_{\bar{\boldsymbol{y}} \in \bar{\mathcal{U}} \cap B_2}\langle  \bar{\boldsymbol{y}}, \widetilde{D} \boldsymbol{g} \rangle.
\end{align*}
To find a high-probability bound on this random variable, we first bound the expectation.
Define $\Pi_r$ to be the orthogonal projection on to the row space of
$\Pi_{\bar{\mathcal{U}}}\widetilde{D}$, so that $\Pi_{\bar{\mathcal{U}}}\widetilde{D} = \Pi_{\bar{\mathcal{U}}}\widetilde{D}\Pi_r$.
Then
\begin{subequations}
\begin{align}
\mathbb{E}\sup_{\boldsymbol{y} \in \mathcal{U} \cap B_2} 2\langle SDF\boldsymbol{y}, \widetilde{D} \boldsymbol{\eta}\rangle 
&\leq \frac{\sigma}{\sqrt{ m }}\frac{8}{3} \mathbb{E} \sup_{\boldsymbol{z} \in \bar{\mathcal{U}} \cap B_2} \langle \boldsymbol{z}, \widetilde{D} \boldsymbol{g}\rangle  \label{eq:op:norm:1}\\
 &= \frac{\sigma}{\sqrt{ m }}\frac{8}{3} \mathbb{E} \| \Pi_{\bar{\mathcal{U}}}\widetilde{D}  \boldsymbol{g}\|_2  \label{eq:op:norm:2}\\
 &= \frac{\sigma}{\sqrt{ m }}\frac{8}{3} \mathbb{E} \| \Pi_{\bar{\mathcal{U}}}\widetilde{D}\Pi_r\boldsymbol{g}\|_2  \label{eq:op:norm:3}\\
&\le  \frac{\sigma}{\sqrt{ m }}\frac{8}{3} \| \Pi_{\bar{\mathcal{U}}}\widetilde{D}\|\mathbb{E} \|\Pi_r\boldsymbol{g}\|_2  \label{eq:op:norm:4}\\
 &\le   \frac{\sigma}{\sqrt{ m }}\frac{8}{3} \| \widetilde{D}\Pi_{\bar{\mathcal{U}}}\| \sqrt{\ell}. \label{eq:op:norm:5}
\end{align}
\end{subequations}
The last inequality follows because the matrix
$\Pi_{\bar{\mathcal{U}}}\widetilde{D}$ has rank $\ell$, and so $\Pi_r$ is
a projection onto a $\ell-$dimensional subspace.
To bound the deviation away from the expectation, we consider the function 
\begin{displaymath}
\boldsymbol{g} \to \frac{2\sigma}{\sqrt{ m }}\sup_{\boldsymbol{y}\in \mathcal{U} \cap B_2} \langle S DF\boldsymbol{y}, \widetilde{D} \boldsymbol{g} \rangle.
\end{displaymath}
With an argument similar to the one we used to bound the expectation, 
we find this function to be Lipschitz with constant 
$\frac{\sigma}{\sqrt{ m }}\frac{8}{3}\|\widetilde{D}\Pi_{\bar{\mathcal{U}}}\|$.
Then by Gaussian concentration (\cite[Theorem 5.5]{maurerConcentrationInequalitiesSubGaussian2021}), we find that for any $t>0$,
\begin{equation}
\label{eq:first:main:ineq}
\sup_{\boldsymbol{y} \in \mathcal{U} \cap B_2} 2\langle SDF\boldsymbol{y}, \widetilde{D} \boldsymbol{\eta}\rangle \leq  \frac{\sigma}{\sqrt{ m }}\frac{8}{3}\|\widetilde{D} \Pi_{\bar{\mathcal{U}}}\| (\sqrt{\ell}+t)
\end{equation}
with probability at least $1- 2\exp(-t^2)$.

In the case where $\mathbb{K}$ is $\mathbb{R}$, we to bound $\lVert \widetilde{D}\Pi_{\bar{\mathcal{U}}}\rVert$ with the following lemma.
\begin{lemma}[Operator norm of a diagonal matrix on an incoherent subspace]
\label{loc:operator_norm_of_a_diagonal_matrix_on_an_incoherent_subspace.statement}
Let $\bar{\mathcal{U}} \subseteq \mathbb{R}^{m}$ be
a non-trivial subspace and let $\boldsymbol{\beta} \in \mathbb{R}_{+}^{m}$ be the local coherences of the canonical basis with respect to $\bar{\mathcal{U}}$, meaning that
$\beta_i:=\sup_{u \in  \bar{\mathcal{U}} \cap B_2} |u_i|$. Let $\widetilde{D}  =  \Diag( \tilde{\boldsymbol{d}})$ for a vector $\tilde{\boldsymbol{d}} \in \mathbb{R}_{++}^m$ with non-increasing entries. Then
\begin{displaymath}
\|\widetilde{D} \proj_{\bar{\mathcal{U}}}\| \le \|\widetilde{D}\trunc(\boldsymbol{\beta})\|_2.
\end{displaymath}
\end{lemma}
For $\boldsymbol{d} \in \mathbb{R}^n$ the diagonal entries of $D$, and a fixed $i \in [m],$ the $i^{\text{th}}$ local coherence of $\bar{\mathcal{U}}$ is 

\begin{subequations}
\begin{align}
\sup_{\boldsymbol{u} \in \bar{\mathcal{U}}  \cap B_2} |\boldsymbol{e}_i^* \boldsymbol{u}|
&= \sup_{\boldsymbol{h} \in \mathcal{U} \cap B_2} \left|\boldsymbol{e}_i^* \frac{ SDF \boldsymbol{h}}{\|SDF \boldsymbol{h}\|_2}\right| \label{eq:embeddedcoherence:1}\\
&\le \sup_{\boldsymbol{h} \in \mathcal{U} \cap B_2} \frac{ (SD \boldsymbol{\alpha})_i }{\|SDF \boldsymbol{h}\|_2} \label{eq:embeddedcoherence:2}\\
&\le \frac{3}{2}  (S D \boldsymbol{\alpha})_i. \label{eq:embeddedcoherence:3}
\end{align}
\end{subequations}
\Cref{eq:embeddedcoherence:2} follows because for any $\boldsymbol{h} \in \mathcal{T} \cap B_2^n$, the vector $DF\boldsymbol{h}$ has entries dominated entry-wise in magnitude by $D\boldsymbol{\alpha}$.  \Cref{eq:embeddedcoherence:3} follows from a lower-bound on the denominator by using the RIP of $SDF$ on $\mathcal{T}$. 
\Cref{loc:operator_norm_of_a_diagonal_matrix_on_an_incoherent_subspace.statement} then tells us
that 
\begin{equation}
\label{eq:bound:op}
\|\widetilde{D}\Pi_{\bar{\mathcal{U}}}\| \le \frac{3}{2}\|\widetilde{D}\trunc(SD \boldsymbol{\alpha})\|_2.
\end{equation}
If, instead, $\mathbb{K}$ is $\mathbb{C}$, we require a slightly modified lemma,
which nonetheless yields the same result.
\begin{lemma}[Soft incoherent projection bound]
\label{loc:soft_incoherent_projection_bound.statement}
Let $\bar{\mathcal{U}} \subseteq \mathbb{R}^{2m}$ be
a non-trivial subspace such that for $i  \in [m]$, $\sup_{\bar{u}  \in \bar{\mathcal{U}}  \cap B_2} (\bar{u}_{2i-1}^2 + \bar{u}_{2i}^2)  \le \beta_i^2$ for a vector 
$\boldsymbol{\beta} \in \mathbb{R}_{++}^m$. Let $\widetilde{D}  =  \Diag( \tilde{\boldsymbol{d}})$ for a vector $\tilde{\boldsymbol{d}} \in \mathbb{R}_{++}^m$ with non-increasing entries. Suppose that the vector $\bar{\boldsymbol{d}} \in \mathbb{R}_{++}^{2m}$ has entries $\bar{d}_{2i-1} = \bar{d}_{2i}  =  \tilde{d}_i \, \forall i  \in [m]$, and let $\bar{D}  =  \Diag(\bar{\boldsymbol{d}})$. Then
\begin{displaymath}
\|\bar{D} \Pi_{\bar{\mathcal{U}}}\|  \le \|\widetilde{D} \trunc(\boldsymbol{\beta})\|_2.
\end{displaymath}
\end{lemma}
To find the relevant local coherences $\boldsymbol{\beta}$ in \Cref{loc:soft_incoherent_projection_bound.statement}, note that \Cref{eq:embeddedcoherence:3} still
holds if we let $F$ be the complex unitary matrix in $\mathbb{C}^{n  \times n}$. Following the same argument,
\begin{displaymath}
\sup_{\boldsymbol{u} \in \bar{\mathcal{U}}  \cap  B_2}|u_i|  \le \frac{3}{2}
(SD \boldsymbol{\alpha})_i,
\end{displaymath}
where $\bar{\mathcal{U}}  \subseteq \mathbb{C}^m$ and $u_i$ is a complex number. Back in $\mathbb{R}^{2m}$, this statement
becomes
\begin{displaymath}
\forall \boldsymbol{u} \in \bar{\mathcal{U}} \cap B_2,\, \,   u_{2i-1}^2 + u_{2i}^2
 \le \frac{9}{4} (SD \boldsymbol{\alpha})_i^2.
\end{displaymath}
Using \Cref{loc:soft_incoherent_projection_bound.statement} with $\beta_i  = \frac{3}{2} (SD \boldsymbol{\alpha})_i$
implies that \Cref{eq:bound:op} also holds in the case that $\mathbb{K}$ is $\mathbb{C}$.

Applying \Cref{eq:bound:op} to \Cref{eq:first:main:ineq}, we find that
\begin{displaymath}
\sup_{\boldsymbol{y} \in \mathcal{U} \cap B_2} 2\langle SDF\boldsymbol{y}, \widetilde{D} \boldsymbol{\eta}\rangle \le 4  \frac{\sigma}{\sqrt{ m }}\|\widetilde{D}\trunc(SD \boldsymbol{\alpha})\|_2\left( \sqrt{\ell} + t \right)
\end{displaymath}
with probability at least $1- 2 \exp(-t^2)$.
With a union bound as described by \Cref{loc:union_bound_on_tails_with_gaussian_tail_variables.statement} over the $M$ subspaces constituting $\mathcal{T}$, it follows that 
\begin{displaymath}
\sup_{\boldsymbol{y} \in \mathcal{T} \cap B_2} 2\langle SDF\boldsymbol{y}, \widetilde{D} \boldsymbol{\eta}\rangle \le 4 \frac{\sigma}{\sqrt{ m }}\|\widetilde{D}\trunc(SD \boldsymbol{\alpha})\|_2\left( \sqrt{\ell} + \sqrt{\log M} + t \right)
\end{displaymath}
with probability at least $1-2 \exp(-t^2)$.
\end{proof}
\begin{proof}[\hypertarget{loc:operator_norm_of_a_diagonal_matrix_on_an_incoherent_subspace.proof}Proof of \Cref{loc:operator_norm_of_a_diagonal_matrix_on_an_incoherent_subspace.statement}]

Note that
\begin{align*}
\|\widetilde{D} \proj_{\bar{\mathcal{U}}}\|^2 &= \sup_{\boldsymbol{u} \in \bar{\mathcal{U}} \cap \sphere{m}} \|\widetilde{D} \boldsymbol{u}\|_2^2 \\
 &= \sup_{\boldsymbol{u} \in \bar{\mathcal{U}} \cap \sphere{m}} \sum_{i = 1}^m u_i^2 \tilde{d}_i^2\\
 &\le \sup_{\boldsymbol{u} \in \sphere{m},  u_i^2  \le \beta_i^2} \sum_{i = 1}^m u_i^2 \tilde{d}_i^2
\end{align*}
Re-parameterize to $\boldsymbol{p} \in \Delta^{n-1}$ with the substitution $p_i  = u_i^2$, and let $\boldsymbol{c}  = \boldsymbol{\tilde{d}}^{.2}$ (squaring the vector element-wise), we find that
\begin{equation}
\label{eq:main:prob:opt}
\|\widetilde{D} \proj_{\bar{\mathcal{U}}}\|^2 \le \max_{\boldsymbol{p} \in \Delta^{m-1}, \, p_i  \le \beta_i^2} \boldsymbol{c}^* \boldsymbol{p}.
\end{equation}
Since $\boldsymbol{\tilde{d}}$ has decreasing entries, the solution to this
problem is the vector which concentrates its mass on the first entries
as much as possible, 
which is $\mathbb{T}(\boldsymbol{\beta})^{.2}$. 

Indeed, for any other
$\boldsymbol{p} \in \Delta^{n-1}$, the difference
$\boldsymbol{p}-\mathbb{T}(\boldsymbol{\beta})^{.2}$ is a vector which
sums to zero, and can only have positive entries outside the support of $\mathbb{T}(\boldsymbol{\beta})^{.2}$ (including also the last non-zero entry of $\mathbb{T}(\boldsymbol{\beta})^{.2}$, although this does not affect the argument). Similarly, negative entries of
the difference vector must occur inside the support of $\mathbb{T}(\boldsymbol{\beta})^{.2}$.
But the objective vector $\boldsymbol{c}$ has larger coefficients on the support of $\mathbb{T}(\boldsymbol{\beta})^{.2}$
than off of it, therefore $\boldsymbol{c}^* (\boldsymbol{p}-\mathbb{T}(\boldsymbol{\beta})^{.2})  \le 0$, which means that $\boldsymbol{c}^* \boldsymbol{p}  \le \boldsymbol{c}^* \mathbb{T}(\boldsymbol{\beta})^{.2}$. It follows that $\mathbb{T}(\boldsymbol{\beta})^{.2}$
maximizes the objective, achieving a value of $\|\widetilde{D}\trunc(\boldsymbol{\beta})\|_2^2$,
which gives us the desired upper-bound.
\end{proof}
\begin{proof}[\hypertarget{loc:soft_incoherent_projection_bound.proof}Proof of \Cref{loc:soft_incoherent_projection_bound.statement}]

Similarly as in \hyperlink{loc:operator_norm_of_a_diagonal_matrix_on_an_incoherent_subspace.proof}{the proof} of
\Cref{loc:operator_norm_of_a_diagonal_matrix_on_an_incoherent_subspace.statement}, we find an
upper-bound for $\|\bar{D} \Pi_{\bar{\mathcal{U}}}\|^2$ to be
\begin{equation}
\label{eq:first:optim}
\max_{\bar{\boldsymbol{p}} \in \Delta^{m-1}, \, \bar{p}_{2i-1} + \bar{p}_{2i} \le \beta_i^2} \bar{\boldsymbol{c}}^* \bar{\boldsymbol{p}}.
\end{equation}
where $\bar{\boldsymbol{c}}  \in \mathbb{R}^{2m}$ has entries $\bar{c}_i  = \bar{d}_i^2$.
Substitute $\bar{\boldsymbol{p}}  \in \Delta^{2m-1}$ by $\boldsymbol{p} \in \Delta^{m-1}$ with $p_i  = \bar{p}_{2i-1} + \bar{p}_{2i}$,
and since $\bar{\boldsymbol{d}}$ has pairs of repeated entries, 
take $\boldsymbol{c}  \in \mathbb{R}^{2m}$ to have entries $c_i  =  \bar{c}_{2i}$. The
constraints $\bar{p}_{2i-1} + \bar{p}_{2i}  \le \beta_i^2$ then becomes
$p_i  \le \beta_i^2$. Therefore,
we find that \Cref{eq:first:optim} equals
\begin{displaymath}
\max_{\boldsymbol{p} \in \Delta^{m-1}, \, p_i  \le \beta_i^2} \boldsymbol{c}^* \boldsymbol{p},
\end{displaymath}
which matches \Cref{eq:main:prob:opt} in \hyperlink{loc:operator_norm_of_a_diagonal_matrix_on_an_incoherent_subspace.proof}{the proof} of
\Cref{loc:operator_norm_of_a_diagonal_matrix_on_an_incoherent_subspace.statement}. Therefore, we
find the same upper-bound as in \Cref{loc:operator_norm_of_a_diagonal_matrix_on_an_incoherent_subspace.statement}.
\end{proof}

\section{Conclusion}
We have analyzed the stochastic noise dependence of variable density sampling in compressed sensing and have shown that optimized sampling leads to de-noising.  We believe this is the first de-noising result in variable density compressive sampling. We assumed the prior belongs to a union of subspaces, thus allowing both sparse and generative compressed sensing models as special cases.  In the latter case, we consider the prior to be the range of a neural net with ReLU activation functions. An open question is whether this work can be extended to other prior models, such as the range of a neural net with smooth activation functions.

\section*{Acknowledgments}
Y. Plan is partially supported by an NSERC Discovery Grant (GR009284), an NSERC Discovery Accelerator Supplement (GR007657), and a Tier II Canada Research Chair in Data Science (GR009243). O. Yilmaz was supported by an NSERC Discovery Grant (22R82411) O. Yilmaz also acknowledges support by the Pacific Institute for the Mathematical Sciences (PIMS) and the CNRS -- PIMS International Research Laboratory. Large language models were used while writing the manuscript for help with grammar and phrasing (Claude, Grok, and Chatgpt).

\bibliographystyle{siamplain}
\bibliography{bibliography}

\begin{thebibliography}{10}

\bibitem{adcockCAS4DLChristoffelAdaptive2022}
{\sc B.~Adcock, J.~M. Cardenas, and N.~Dexter}, {\em {{CAS4DL}}:
  {{Christoffel}} adaptive sampling for function approximation via deep
  learning}, Sampling Theory, Signal Processing, and Data Analysis, 20 (2022),
  p.~21, \url{https://doi.org/10.1007/s43670-022-00040-8}.

\bibitem{adcockUnifiedFrameworkLearning2023}
{\sc B.~Adcock, J.~M. Cardenas, and N.~Dexter}, {\em A unified framework for
  learning with nonlinear model classes from arbitrary linear samples}, Nov.
  2023, \url{https://doi.org/10.48550/arXiv.2311.14886},
  \url{https://arxiv.org/abs/2311.14886}.

\bibitem{adcockCompressiveImagingStructure2021}
{\sc B.~Adcock and A.~C. Hansen}, {\em Compressive {{Imaging}}: {{Structure}},
  {{Sampling}}, {{Learning}}}, Cambridge University Press, Cambridge, 2021,
  \url{https://doi.org/10.1017/9781108377447}.

\bibitem{adcockBreakingCoherenceBarrier2017a}
{\sc B.~Adcock, A.~C. Hansen, C.~Poon, and B.~Roman}, {\em Breaking the
  {{Coherence Barrier}}: {{A New Theory}} for {{Compressed Sensing}}}, Forum of
  Mathematics, Sigma, 5 (2017), \url{https://doi.org/10.1017/fms.2016.32}.

\bibitem{axlerLinearAlgebraDone2024}
{\sc S.~Axler}, {\em Linear {{Algebra Done Right}}}, Undergraduate {{Texts}} in
  {{Mathematics}}, Springer International Publishing, Cham, 2024,
  \url{https://doi.org/10.1007/978-3-031-41026-0}.

\bibitem{berkCoherenceParameterCharacterizing2022}
{\sc A.~Berk, S.~Brugiapaglia, B.~Joshi, Y.~Plan, M.~Scott, and {\"O}.~Yilmaz},
  {\em A coherence parameter characterizing generative compressed sensing with
  {{Fourier}} measurements}, IEEE Journal on Selected Areas in Information
  Theory,  (2022), pp.~1--1, \url{https://doi.org/10.1109/JSAIT.2022.3220196}.

\bibitem{berkModeladaptedFourierSampling2023}
{\sc A.~Berk, S.~Brugiapaglia, Y.~Plan, M.~Scott, X.~Sheng, and O.~Yilmaz},
  {\em Model-adapted {{Fourier}} sampling for generative compressed sensing},
  in {{NeurIPS}} 2023 {{Workshop}} on {{Deep Learning}} and {{Inverse
  Problems}}, Nov. 2023.

\bibitem{boraCompressedSensingUsing2017}
{\sc A.~Bora, A.~Jalal, E.~Price, and A.~G. Dimakis}, {\em Compressed
  {{Sensing}} using {{Generative Models}}}, in Proceedings of the 34th
  {{International Conference}} on {{Machine Learning}}, PMLR, July 2017,
  pp.~537--546.

\bibitem{candesSparsityIncoherenceCompressive2007}
{\sc E.~Candes and J.~Romberg}, {\em Sparsity and {{Incoherence}} in
  {{Compressive Sampling}}}, Inverse Problems, 23 (2007), pp.~969--985,
  \url{https://doi.org/10.1088/0266-5611/23/3/008},
  \url{https://arxiv.org/abs/math/0611957}.

\bibitem{candesRobustUncertaintyPrinciples2006}
{\sc E.~Candes, J.~Romberg, and T.~Tao}, {\em Robust uncertainty principles:
  Exact signal reconstruction from highly incomplete frequency information},
  IEEE Transactions on Information Theory, 52 (2006), pp.~489--509,
  \url{https://doi.org/10.1109/TIT.2005.862083}.

\bibitem{candesDecodingLinearProgramming2005}
{\sc E.~Candes and T.~Tao}, {\em Decoding by linear programming}, IEEE
  Transactions on Information Theory, 51 (2005), pp.~4203--4215,
  \url{https://doi.org/10.1109/TIT.2005.858979}.

\bibitem{candesNearOptimalSignalRecovery2006}
{\sc E.~J. Candes and T.~Tao}, {\em Near-{{Optimal Signal Recovery From Random
  Projections}}: {{Universal Encoding Strategies}}?}, IEEE Transactions on
  Information Theory, 52 (2006), pp.~5406--5425,
  \url{https://doi.org/10.1109/TIT.2006.885507}.

\bibitem{cardenasCS4MLGeneralFramework2023}
{\sc J.~M. Cardenas, B.~Adcock, and N.~Dexter}, {\em {{CS4ML}}: {{A}} general
  framework for active learning with arbitrary data based on {{Christoffel}}
  functions}, Advances in Neural Information Processing Systems, 36 (2023),
  pp.~19990--20037.

\bibitem{cohenCompressedSensingBest2009}
{\sc A.~Cohen, W.~Dahmen, and R.~Devore}, {\em Compressed sensing and best k
  -term approximation}, Journal of the American Mathematical Society, 22
  (2009), pp.~211--231, \url{https://doi.org/10.1090/S0894-0347-08-00610-3}.

\bibitem{donohoCompressedSensing2006}
{\sc D.~Donoho}, {\em Compressed {{Sensing}}}, Information Theory, IEEE
  Transactions on, 52 (2006), pp.~1289--1306,
  \url{https://doi.org/10.1109/TIT.2006.871582}.

\bibitem{foucartMathematicalIntroductionCompressive2013}
{\sc S.~Foucart and H.~Rauhut}, {\em A {{Mathematical Introduction}} to
  {{Compressive Sensing}}}, Springer New York, June 2013.

\bibitem{khvostovaAlgorithmMachineCalculation1965}
{\sc E.~KHvostova}, {\em An {{Algorithm}} for the {{Machine Calculation}} of
  {{Complex Fourier Series}}}, Mathematics of Computation,  (1965).

\bibitem{krahmerLocalCoherenceSampling2013}
{\sc F.~Krahmer, H.~Rauhut, and R.~Ward}, {\em Local coherence sampling in
  compressed sensing}, Proceedings of the 10th International Conference on
  Sampling Theory and Applications,  (2013).

\bibitem{krahmerStableRobustSampling2014}
{\sc F.~Krahmer and R.~Ward}, {\em Stable and {{Robust Sampling Strategies}}
  for {{Compressive Imaging}}}, IEEE Transactions on Image Processing, 23
  (2014), pp.~612--622, \url{https://doi.org/10.1109/TIP.2013.2288004}.

\bibitem{lustigSparseMRIApplication2007}
{\sc M.~Lustig, D.~Donoho, and J.~M. Pauly}, {\em Sparse {{MRI}}: {{The}}
  application of compressed sensing for rapid {{MR}} imaging}, Magnetic
  Resonance in Medicine, 58 (2007), pp.~1182--1195,
  \url{https://doi.org/10.1002/mrm.21391}.

\bibitem{maurerConcentrationInequalitiesSubGaussian2021}
{\sc A.~Maurer and M.~Pontil}, {\em Concentration inequalities under
  sub-{{Gaussian}} and sub-exponential conditions}, in Advances in {{Neural
  Information Processing Systems}}, vol.~34, Curran Associates, Inc., 2021,
  pp.~7588--7597.

\bibitem{naderiIndependentMeasurementsGeneral2021}
{\sc A.~Naderi and Y.~Plan}, {\em Beyond {{Independent Measurements}}:
  {{General Compressed Sensing}} with {{GNN Application}}}, Oct. 2021,
  \url{https://doi.org/10.48550/arXiv.2111.00327},
  \url{https://arxiv.org/abs/2111.00327}.

\bibitem{puyVariableDensityCompressive2011}
{\sc G.~Puy, P.~Vandergheynst, and Y.~Wiaux}, {\em On {{Variable Density
  Compressive Sampling}}}, IEEE Signal Processing Letters, 18 (2011),
  pp.~595--598, \url{https://doi.org/10.1109/LSP.2011.2163712},
  \url{https://arxiv.org/abs/1109.6202}.

\bibitem{rauhutCompressiveSensingStructured2010}
{\sc H.~Rauhut}, {\em Compressive {{Sensing}} and {{Structured Random
  Matrices}}}, De Gruyter,  (2010).

\bibitem{rauhutSparseLegendreExpansions2012}
{\sc H.~Rauhut and R.~Ward}, {\em Sparse {{Legendre}} expansions via
  {$\ell$}1-minimization}, Journal of Approximation Theory, 164 (2012),
  pp.~517--533, \url{https://doi.org/10.1016/j.jat.2012.01.008}.

\bibitem{rudelsonSparseReconstructionFourier2008}
{\sc M.~Rudelson and R.~Vershynin}, {\em On sparse reconstruction from
  {{Fourier}} and {{Gaussian}} measurements}, Communications on Pure and
  Applied Mathematics, 61 (2008), pp.~1025--1045,
  \url{https://doi.org/10.1002/cpa.20227}.

\bibitem{schnassDictionaryPreconditioningGreedy2008}
{\sc K.~Schnass and P.~Vandergheynst}, {\em Dictionary {{Preconditioning}} for
  {{Greedy Algorithms}}}, IEEE Transactions on Signal Processing, 56 (2008),
  pp.~1994--2002, \url{https://doi.org/10.1109/TSP.2007.911494}.

\bibitem{vandenbergProbingParetoFrontier2009}
{\sc E.~Van Den~Berg and M.~P. Friedlander}, {\em Probing the {{Pareto
  Frontier}} for {{Basis Pursuit Solutions}}}, SIAM Journal on Scientific
  Computing, 31 (2009), pp.~890--912, \url{https://doi.org/10.1137/080714488}.

\bibitem{vershyninHighDimensionalProbabilityIntroduction2018}
{\sc R.~Vershynin}, {\em High-{{Dimensional Probability}}: {{An Introduction}}
  with {{Applications}} in {{Data Science}}}, Cambridge University Press, Sept.
  2018.

\bibitem{xiangli2020realrealquestion}
{\sc Y.~Xiangli, Y.~Deng, B.~Dai, C.~C. Loy, and D.~Lin}, {\em Real or not
  real, that is the question}, 2020, \url{https://arxiv.org/abs/2002.05512},
  \url{https://arxiv.org/abs/2002.05512}.

\end{thebibliography}
\appendix
\section{RIP of non-uniformly subsampled matrices}
\label{loc:body.proofs.rip_of_non:uniformly_subsampled_matrices}
We first show a version of \hyperlink{loc:rip_of_with:replacement_non:uniform_sampling_matrix_on_a_union_of_subspaces.proof}{the proof} of \Cref{loc:rip_of_with:replacement_non:uniform_sampling_matrix_on_a_union_of_subspaces.statement} for a single subspace.
\begin{lemma}[Deviation of the CS matrix on a subspace]
\label{loc:deviation_of_with:replacement_matrix_on_subspace.statement}
Let $F \in \measfield^{n \times n}$ be a unitary matrix and $S \in \mathbb{R}^{m \times n}$
a with-replacement sampling matrix associated with a probability vector $\boldsymbol{p} \in (0,1]^n \cap \Delta^{n-1}$. 
Consider the diagonal matrix $D = \Diag\left(\boldsymbol{d}\right)$ where $d_i = (n p_i)^{-1/2}$.
For a subspace $\mathcal{U} \subseteq \mathbb{R}^n$ of dimension at most $\ell$, let
$\boldsymbol{\alpha}$ be the local coherences of $F$ with respect to $\mathcal{U}$.
With $\mu(\boldsymbol{\alpha}, \boldsymbol{p})$ denoting the complexity function from \Cref{loc:with:replacement_complexity.statement}, 
we have for any $t > 0$:
\begin{displaymath}
\sup_{\boldsymbol{x} \in \mathcal{U}\cap \sphere{n}} \left|\|SDF\boldsymbol{x}\|_{2}-1 \right|  \lesssim \frac{\mu(\boldsymbol{\alpha}, \boldsymbol{p})}{\sqrt{ m}}\sqrt{\log \ell} + \frac{\mu(\boldsymbol{\alpha}, \boldsymbol{p})}{ \sqrt{m}} t
\end{displaymath}
with probability at least $1-2\exp(-t^2)$.
\end{lemma}
\begin{proof}[\hypertarget{loc:deviation_of_with:replacement_matrix_on_subspace.proof}Proof of \Cref{loc:deviation_of_with:replacement_matrix_on_subspace.statement}]

In what follows, we use the fact that $\forall \boldsymbol{x} \in  \mathcal{U}$, $SDF\boldsymbol{x} = SDFP_\mathcal{U}^* P_\mathcal{U}\boldsymbol{x}$ where $P_\mathcal{U} \in \mathbb{R}^{\ell \times n}$ is the matrix with rows that are a fixed orthonormal basis of $\mathcal{U}$. 
This holds because $P_\mathcal{U}^* P_\mathcal{U} = \proj_\mathcal{U} \in \mathbb{R}^{n \times n}$ and $\boldsymbol{x} \in \mathcal{U}$. Consider
\begin{subequations}
\begin{align}
(\star) :=\sup_{\boldsymbol{x} \in \mathcal{U} \cap \sphere{n}} \left|\left\| SDF\boldsymbol{x}\right\|_{2}^2 - 1 \right|
 & = \sup_{\boldsymbol{u} \in \mathcal{U} \cap \sphere{n}} \left| \left\| S D F P_\mathcal{U}^* P_\mathcal{U} \boldsymbol{u}\right\|_{2}^2 - 1 \right| \label{eq:first:matrix:1}\\
 & = \sup_{\boldsymbol{x} \in \mathbb{R}^\ell \cap \sphere{\ell}} \left| \left\| S D F P_\mathcal{U}^* \boldsymbol{x}\right\|_{2}^2 - 1 \right|                                   \label{eq:first:matrix:2}\\
 & = \sup_{\boldsymbol{x} \in \mathbb{R}^\ell \cap \sphere{\ell}} \left| \boldsymbol{x}^* \left[ (S D F P_\mathcal{U}^*)^* (S D F P_\mathcal{U}^*) - I \right]\boldsymbol{x} \right|. \label{eq:first:matrix:3}
\end{align}
\end{subequations}
\Cref{eq:first:matrix:2} follows from a change of variables $\boldsymbol{x} = P_\mathcal{U} \boldsymbol{u}$. 
The matrix in the square brackets is Hermitian, and therefore
by \cite[Result 7.15]{axlerLinearAlgebraDone2024}, 
$\boldsymbol{x}^* [\ldots] \boldsymbol{x}$ is a real number. We can
therefore take the real part of the Hermitian matrix:
\begin{subequations}
\begin{align}
(\star)&=\sup_{\boldsymbol{x} \in \mathbb{R}^\ell \cap \sphere{\ell}} \left| \boldsymbol{x}^* \sum_{i=1}^{m} \mathcal{R}\left[P_\mathcal{U} F^* D \boldsymbol{s}_i \boldsymbol{s}_i^* D F P_\mathcal{U}^* - \frac{1}{m}I\right]\boldsymbol{x} \right|\\
& \le  \left\| \sum_{i =1}^m \left(\mathcal{R}\left[ \boldsymbol{v}_i \boldsymbol{v}_i^* \right] - \frac{1}{m}I\right)\right\|
\end{align}
\end{subequations}
for the random vectors $\boldsymbol{v}_i := P_\mathcal{U} F^* D \boldsymbol{s}_i$.
This is a sum of i.i.d. $\ell  \times  \ell$ real random matrices because $S$ 
has i.i.d. rows. We now introduce the central ingredient of this
proof: the Matrix Bernstein inequality~\cite[Theorem 5.4.1]{vershyninHighDimensionalProbabilityIntroduction2018}.
\begin{lemma}[Matrix Bernstein]
\label{loc:matrix_bernstein.statement}
Let $X_1, ..., X_N$ be independent, mean zero, symmetric random matrices in
$\mathbb{R}^{\ell \times \ell}$, such that $||X_i|| \leq K$ almost surely for
all $i \in [N]$. Then, for every $t\geq 0$, we have
\begin{displaymath}
\mathbb{P} \left\{ \left\lVert \sum_{i=1}^m X_i \right\rVert \geq t \right\} \leq 2\ell\exp \left( - \frac{t^2/2}{\sigma^2 + Kt/3}\right),
\end{displaymath}
where $\sigma^2=\lVert\sum_{i=1}^m \mathbb{E}X_i^2\rVert$.
\end{lemma}
The random vectors $\{ \boldsymbol{v}_i \}_{i \in [m]}$ have two key properties.
First, the real parts of their outer products are \emph{isotropic} (up to a scalar multiplication). 
Indeed, for any fixed $i \in [n]$,
\begin{align*}
\mathbb{E}[\mathcal{R}(\boldsymbol{v}_i \boldsymbol{v}_i^*)] & = \mathcal{R}\mathbb{E} [P_\mathcal{U} F^* D \boldsymbol{s}_i \boldsymbol{s}_i^* D F P_\mathcal{U}^* ]    \\
& = \mathcal{R}( P_\mathcal{U} F^* D \mathbb{E}[\boldsymbol{s}_i \boldsymbol{s}_i^*] D F P_\mathcal{U}^*)\\
& = \mathcal{R}\left( P_\mathcal{U} F^* \left[\sum_{j=1}^n \frac{1}{np_j} p_j\frac{n}{m} \boldsymbol{e}_j \boldsymbol{e}_j^\star \right] F P_\mathcal{U}^* \right) \\
&= \frac{1}{m}I.
\end{align*}
The isotropic property gives us immediately that, as required by \Cref{loc:matrix_bernstein.statement}, the matrices $\mathcal{R}\left[ \boldsymbol{v}_i \boldsymbol{v}_i^* - \frac{1}{m}I \right]$ are mean-zero.

The second property of the vectors $\boldsymbol{v}_i$ is a bound on $\sup_{\boldsymbol{x} \in \mathbb{R}^\ell  \cap B_2}|\langle \boldsymbol{x}, \boldsymbol{v}_i\rangle|$:
\begin{subequations}
\begin{align}
\sup_{\boldsymbol{x} \in \mathbb{R}^\ell  \cap B_2}|\langle \boldsymbol{x}, \boldsymbol{v}_i\rangle| & = \sup_{\boldsymbol{x} \in \mathbb{R}^\ell  \cap B_2}|\langle \boldsymbol{x},P_\mathcal{U} (F^* D \boldsymbol{s}_i)\rangle| \label{eq:coherence:mu:1}\\
& = \frac{1}{\sqrt{ m }} \max_{j \in  [n]}\frac{1}{\sqrt{ p_j}}\sup_{\boldsymbol{u} \in \mathcal{U} \cap B_2} |\left\langle \boldsymbol{u}, \boldsymbol{f}_j \right\rangle|  \label{eq:coherence:mu:2}\\
&  =  \frac{\mu(\boldsymbol{\alpha}, \boldsymbol{p})}{\sqrt{ m }}, \label{eq:coherence:mu:3}
\end{align}
\end{subequations}
where $\boldsymbol{\alpha}$ is the local coherence vector of $F$ with respect to $\mathcal{U}$.
To be concise, let $\mu := \mu(\boldsymbol{\alpha}, \boldsymbol{p})$.
We proceed to compute a value for $K$. By triangle inequality and property of the operator norm of rank one matrices, we see that
\begin{displaymath}
\left\|\mathcal{R}\left[\boldsymbol{v}_i \boldsymbol{v}_i^* - \frac{1}{m}I\right]\right\| \le  \sup_{\boldsymbol{x} \in \mathbb{R}^\ell  \cap B_2} |\langle \boldsymbol{x}, \boldsymbol{v}_i\rangle|^2 + \frac{1}{m} \le 2\frac{\mu^2}{m}.
\end{displaymath}
The last inequality holds because of the lower bound $\mu^2 \ge 1$ which we now
justify. Consider that from \Cref{loc:optimize_the_with:replacement_sampling_probabilities.statement}
we have that $\mu \ge \|\boldsymbol{\alpha}\|_2$, and
furthermore that any non-empty prior set contains a unit vector
$\hat{\boldsymbol{u}} \in \mathbb{R}^n$, and so $\|\boldsymbol{\alpha}\|_2 \geq \|F \hat{\boldsymbol{u}}\|_2 = 1$. This gives us the desired lower bound by monotonicity of $\mu$ over set containment.

We now compute $\sigma^2$.
\begin{subequations}
\begin{align}
\sigma^2 & = \left\|\sum_{i=1}^m \mathbb{E}\left[\mathcal{R}\left( \boldsymbol{v}_i \boldsymbol{v}_i^* - \frac{1}{m} I \right)^2\right] \right \|  \label{eq:var:1}\\
&= \sup_{\boldsymbol{x} \in \mathbb{R}^\ell \cap B_2} \boldsymbol{x}^* \sum_{i=1}^m \left( \mathbb{E}[\mathcal{R}(\boldsymbol{v}_i \boldsymbol{v}_i^*) \mathcal{R}(\boldsymbol{v}_i \boldsymbol{v}_i^*)]-\frac{1}{m}I \right) \boldsymbol{x}. \label{eq:var:2}\\
 &= \sup_{\boldsymbol{x} \in \mathbb{R}^\ell \cap B_2}  \sum_{i=1}^m \left( \mathbb{E}[\boldsymbol{x}^*\mathcal{R}(\boldsymbol{v}_i \boldsymbol{v}_i^*) \mathcal{R}(\boldsymbol{v}_i \boldsymbol{v}_i^*)\boldsymbol{x}]-\frac{1}{m}\right). \label{eq:var:3}
\end{align}
\end{subequations}
\Cref{eq:var:2} holds because the matrix is symmetric positive semi-definite.
We introduce the unit vector $\hat{\boldsymbol{y}}$ to be the normalization of $\boldsymbol{y} := \mathcal{R}[\boldsymbol{v}_i \boldsymbol{v}_i^*]\boldsymbol{x}$,
and
\begin{subequations}
\begin{align}
\boldsymbol{x}^*  \mathcal{R}[\boldsymbol{v}_i \boldsymbol{v}_i^*]\mathcal{R}[\boldsymbol{v}_i \boldsymbol{v}_i^*]\boldsymbol{x}  &= \boldsymbol{x}^*  \mathcal{R}[\boldsymbol{v}_i \boldsymbol{v}_i^*]\hat{\boldsymbol{y}}\hat{\boldsymbol{y}}^* \mathcal{R}[\boldsymbol{v}_i \boldsymbol{v}_i^*]\boldsymbol{x} \label{eq:mat:var:1}\\
&=  \mathcal{R}[\boldsymbol{x}^*  \boldsymbol{v}_i \boldsymbol{v}_i^*\hat{\boldsymbol{y}}]\mathcal{R}[\hat{\boldsymbol{y}}^* \boldsymbol{v}_i \boldsymbol{v}_i^*\boldsymbol{x}] \label{eq:mat:var:2}\\
&\le |\boldsymbol{x}^*  \boldsymbol{v}_i|^2 |\boldsymbol{v}_i^*\hat{\boldsymbol{y}}|^2 \label{eq:mat:var:3}\\
& \le |\boldsymbol{x}^* \boldsymbol{v}_i|^2 \frac{\mu^2}{m}. \label{eq:mat:var:4}
\end{align}
\end{subequations}
\Cref{eq:mat:var:4} holds because of \Cref{eq:coherence:mu:2}.
With this bound, we find that
\begin{subequations}
\begin{align}
\sigma^2 &\le  \sup_{\boldsymbol{x} \in \mathbb{R}^\ell \cap B_2} \sum_{i=1}^m \mathbb{E}[\boldsymbol{x}^* \boldsymbol{v}_i \boldsymbol{v}_i^* \boldsymbol{x}] \frac{\mu^2}{m} - \frac{1}{m}\sum_{i=1}^m 1 \label{eq:second:var:1}\\
& \le \sup_{\boldsymbol{x} \in \mathbb{R}^\ell \cap B_2} \boldsymbol{x}^* \left(\sum_{i=1}^m \frac{I}{m} \frac{\mu^2}{m} \right) \boldsymbol{x}  \label{eq:second:var:2}\\
& \le  \frac{\mu^2}{m}. \label{eq:second:var:3}
\end{align}
\end{subequations}

\Cref{eq:second:var:2} is obtained by dropping the second term, which is negative.

Then applying the Matrix Bernstein yields
\begin{displaymath}
\mathbb{P}\left\{ \sup_{\boldsymbol{x} \in \mathcal{U} \cap \sphere{n}} \left| \left\| SDF\boldsymbol{x}\right\|_{2}^2 - 1 \right| \ge  t \right\}
\leq 2 \ell \exp\left( -\frac{t^2 /2}{\frac{\mu^2}{m} + \frac{2\mu^2}{m} \frac{t}{3}} \right).
\end{displaymath}
We would like to get our result in terms of the $l_{2}$ norm without the square. For this purpose we make use of the ``square-root trick" that can be found in \cite[Theorem 3.1.1]{vershyninHighDimensionalProbabilityIntroduction2018}. We re-write the above as
\begin{displaymath}
\mathbb{P}\left\{ \sup_{\boldsymbol{x} \in \mathcal{U} \cap \sphere{n}} \left| \left\| SDF\boldsymbol{x}\right\|_{2}^2 - 1 \right| \ge t \right\}
\leq 2\ell \exp\left( - C \frac{m}{\mu^2} \min\left(t^2, t\right) \right).
\end{displaymath}
We make the substitution $t \to \max(\delta, \delta^2)$, which yields
\begin{displaymath}
\mathbb{P}\left \{ \sup_{\boldsymbol{x} \in \mathcal{U} \cap \sphere{n}} \left| \|SDF \boldsymbol{x}\|_2^2 - 1 \right| \geq \max(\delta, \delta^2) \right \} \leq 2\ell \exp\left( -C \frac{m\delta^2 }{\mu^2} \right).
\end{displaymath}
With the restricted inequality $\forall  a,\delta >0,|a -1| \ge \delta \implies |a^2 -1| \ge \max(\delta, \delta^2)$, we infer that
\begin{align*}
\mathbb{P}\left\{ \sup_{\boldsymbol{x} \in \mathcal{U} \cap \sphere{n}} \left| \|SDF\boldsymbol{x}\|_{2} - 1 \right|  \geq \delta \right\} &\le  \mathbb{P}\left \{ \sup_{\boldsymbol{x} \in \mathcal{U} \cap \sphere{n}} \left| \|SDF \boldsymbol{x}\|_2^2 - 1 \right| \geq \max(\delta, \delta^2) \right \}\\
&\leq 2\ell \exp\left( -C \frac{m\delta^2}{\mu^2} \right).
\end{align*}
Finally, with another substitution $\frac{cm\delta^2}{\mu^2} - \log \ell = t^2$, reformulate this bound as
\begin{displaymath}
\sup_{\boldsymbol{x} \in \mathcal{U}\cap \sphere{n}} \left|\|SDF\boldsymbol{x}\|_{2}-1 \right|  \lesssim \frac{\mu}{\sqrt{ m}}\sqrt{\log \ell} + \frac{\mu}{ \sqrt{m}} t
\end{displaymath}
with probability at least $1 - 2\exp(-t^2)$.

---
\end{proof}
The proof of \Cref{loc:rip_of_with:replacement_non:uniform_sampling_matrix_on_a_union_of_subspaces.statement} then follows from a union bound on all the subspaces containing differences in the prior set.
\begin{proof}[\hypertarget{loc:rip_of_with:replacement_non:uniform_sampling_matrix_on_a_union_of_subspaces.proof}Proof of \Cref{loc:rip_of_with:replacement_non:uniform_sampling_matrix_on_a_union_of_subspaces.statement}]

We denote $\mathcal{T}= \cup_{i=1}^M \mathcal{U}_i$ for the set of subspaces $\{ \mathcal{U}_i \}_{i \in [m]}$ in $\measfield^n$ each of dimension no more than $\ell$, and let $\mathcal{U} \subseteq \mathcal{T}$ be any one of these subspaces. 
Then the local coherence vector $\boldsymbol{\alpha}_{\mathcal{U}}$ of $F$ with respect to $\mathcal{U}$ is dominated entry-wise by the local coherences $\boldsymbol{\alpha}$ of $F$ with respect to $\mathcal{T}$, because $\mathcal{U} \subseteq \mathcal{T}$.
Then by 
\Cref{loc:deviation_of_with:replacement_matrix_on_subspace.statement},
it follows that
\begin{equation}
\label{eq:ub:t}
\sup_{\boldsymbol{x} \in \mathcal{U}\cap \sphere{n}} \left|\|SDF\boldsymbol{x}\|_{2}-1 \right|  \lesssim \frac{\mu(\boldsymbol{\alpha}, \boldsymbol{p})}{\sqrt{ m}}\sqrt{\log \ell} + \frac{\mu(\boldsymbol{\alpha}, \boldsymbol{p})}{ \sqrt{m}} t
\end{equation}
with probability at least $1-2\exp(-t^2)$. This is almost the same statement as that of \Cref{loc:deviation_of_with:replacement_matrix_on_subspace.statement} except for the fact that on the r.h.s. we have $\boldsymbol{\alpha}$ instead of $\boldsymbol{\alpha}_{\mathcal{U}}$. This is not a problem because of the monotonicity of $\mu$ in its first argument. 

Note that we have in \Cref{eq:ub:t} an identical statement that applies for each subspace $\mathcal{U}$ composing $\mathcal{T}$. We perform a union bound over all such statements in the manner described by \Cref{loc:union_bound_on_tails_with_gaussian_tail_variables.statement}, and find that 
\begin{displaymath}
\sup_{\boldsymbol{x} \in \mathcal{T} \cap \sphere{n}} \left|\|SDF\boldsymbol{x}\|_{2}-1 \right|  \lesssim \frac{\mu(\boldsymbol{\alpha}, \boldsymbol{p})}{\sqrt{ m}}\sqrt{\log \ell} + \frac{\mu(\boldsymbol{\alpha}, \boldsymbol{p})}{ \sqrt{m}} \sqrt{ \log M }+ \frac{\mu(\boldsymbol{\alpha}, \boldsymbol{p})}{ \sqrt{m}} t
\end{displaymath}
with probability at least $1-2\exp(-t^2)$.
\end{proof}
\section{Proofs by union bounds}
\label{loc:appendix.proofs_by_union_bounds}
Both \Cref{loc:optimized_with:replacement_cs_on_union_of_subspaces.statement}
and \Cref{loc:cs_with_replacement_and_with_denoising_on_unions_of_subspaces.statement}
follow from union bounds on a few lemmas.
\begin{proof}[\hypertarget{loc:optimized_with:replacement_cs_on_union_of_subspaces.proof}Proof of \Cref{loc:optimized_with:replacement_cs_on_union_of_subspaces.statement}]

Let $t_1>0$, and let
\begin{displaymath}
m \gtrsim \lVert \boldsymbol{\alpha}\rVert_{2}^2 \left( \log \ell+ \log M + t_1^2\right).
\end{displaymath}
Each of the following statements holds individually with probability at least $1- 2\exp(-t^2)$ for a variable $t>0$ defined within each of the three results.
\begin{enumerate}
\item The matrix $SDF$ has the RIP thanks to \Cref{loc:rip_of_with:replacement_non:uniform_sampling_matrix_on_a_union_of_subspaces.statement}.
\item When \emph{1.} is satisfied, the recovery error is bounded as specified by \Cref{loc:signal_recovery_with_subsampled_unitary_matrix_with_gaussian_noise_on_union_of_subspaces.statement}.
\item When \emph{1.} is satisfied, the noise sensitivity $\|\widetilde{D}\trunc(SD \boldsymbol{\alpha})\|_2$ in the recovery error bound of~\Cref{loc:signal_recovery_with_subsampled_unitary_matrix_with_gaussian_noise_on_union_of_subspaces.statement} is bounded thanks to~\Cref{loc:tail_on_the_noise_sensitivity_for_adapted_with:replacement_sampling.statement}.
\end{enumerate}

We distinguish between the variables $t$ used within each of the three statements by re-labelling them $t_{1}, t_{2}, t_{3}$ respectively. For some $\delta>0$, let $2\exp(-t_1^2) = \frac{1}{10}\delta$, $2\exp(-t_{2}^2) = \frac{1}{10}\delta$, and $t_{3} = \frac{8}{10}\delta$. Then $t_{1} = t_2 = \sqrt{ \log \frac{20}{\delta} }$. The required statement then holds with probability at least $1-\delta$ from a union bound on the three statements above for this choice of $t_1, t_2, t_3$.
\end{proof}
The following proof is similar.
\begin{proof}[\hypertarget{loc:cs_with_replacement_and_with_denoising_on_unions_of_subspaces.proof}Proof of \Cref{loc:cs_with_replacement_and_with_denoising_on_unions_of_subspaces.statement}]

Let $t_1>0$, and
\begin{displaymath}
m \gtrsim \mu^2(\boldsymbol{\alpha}, \boldsymbol{p})(\log (\ell) + \log M + t_1^2).
\end{displaymath}
Then each of the following statements holds individually with probability at least $1- 2\exp(-t^2)$ for a variable $t>0$ defined within each of the results.

\begin{enumerate}
\item The matrix $SDF$ has the RIP on $\mathcal{T}$ thanks to \Cref{loc:rip_of_with:replacement_non:uniform_sampling_matrix_on_a_union_of_subspaces.statement}.
\item When \emph{1.} is satisfied, the recovery error is bounded as specified by \Cref{loc:signal_recovery_with_subsampled_unitary_matrix_with_gaussian_noise_on_union_of_subspaces.statement}.
\end{enumerate}

We distinguish between the variables $t$ used within each of the two statements by
re-labelling them $t_{1}, t_{2}$ respectively. For some $\delta>0$, let
$2\exp(-t_1^2) = \frac{1}{2}\delta$ and $2\exp(-t_{2}^2) = \frac{1}{2}\delta$.
Then $t_{1} = t_2 = \sqrt{ \log \frac{4}{\delta} }$. The fact that the second
statement is conditional on the success of the first only lessens the true
probability of failure, and so the probability of failure is no more than
$\frac{1}{2}\delta + \frac{1}{2}\delta = \delta$. The result follows.
\end{proof}
We present below a technical lemma for general union bounds
on statements involving sub-Gaussian random variables.
\begin{proposition}[Union bound on sub-Gaussian tail bounds]
\label{loc:union_bound_on_tails_with_gaussian_tail_variables.statement}
Let $\{ S_i(t) \}_{i \in [\ell]}$ be an array of random statements such that $\forall i \in [\ell], t>0,$ $S_i(t)$ is true with probability at least $1-2\exp(-t^2)$. Then with probability at least $1-2exp(-t^2)$, the following statement holds.
\begin{displaymath}
S_i(t + \sqrt{ \log \ell }) \text{ is true }\forall  i \in  [\ell].
\end{displaymath}
\end{proposition}
\begin{proof}[\hypertarget{loc:union_bound_on_tails_with_gaussian_tail_variables.proof}Proof of \Cref{loc:union_bound_on_tails_with_gaussian_tail_variables.statement}]

Let $\gamma>0$.
\begin{align*}
\mathbb{P}(\exists i \text{  s.t. } S_i(\gamma) \text{ is false})
&=\mathbb{P}(\cup_{i \in [\ell]} \{  S_i(\gamma) \text{ is false}\})\\
&\leq  \ell 2\exp(-\gamma^2 )\\
&= 2\exp(-\gamma^2 + \log \ell)
\end{align*}
We perform the substitution $\gamma = \gamma' + \sqrt{ \log \ell }$.
\begin{align*}
\mathbb{P}\left(\exists i \text{ s.t. }S_i\left(\gamma' + \sqrt{ \log \ell }\right) \text{ is false}\right) &\leq 2 \exp(- (\gamma' + \sqrt{ \log \ell })^2 +  \log \ell )\\
&=  2 \exp(- \gamma'^2 - 2 \gamma' \sqrt{ \log \ell })\\
&\leq 2 \exp(-\gamma'^2).
\end{align*}
\end{proof}
\section{Gaussian noise corollary from deterministic noise}
\label{loc:appendix.gaussian_noise_corollary_from_deterministic_noise}
The following is a compressed sensing result with Gaussian measurement noise,
derived as a corollary of a result with deterministic noise (\cite[Theorem 3.6]{adcockUnifiedFrameworkLearning2023}).
\begin{corollary}[Gaussian noise CS as corollary of deterministic noise]
\label{loc:noise_corollary_from_deterministic_results_on_union_of_subspaces.statement}
Consider a union of $N$ $r-$dimensional subspaces $\mathcal{Q} \subseteq \mathbb{R}^n$. Given a unitary matrix $F  \in \mathbb{C}^{n \times n}$ with local coherences $\boldsymbol{\alpha} \in \mathbb{R}_{++}^n$ with respect to the set $\mathcal{Q}-\mathcal{Q}$, take $S$ as a \hyperref[loc:sampling_matrix_with_replacement.statement]{with-replacement sampling matrix} with probability vector $\boldsymbol{p}' =  \left\{\frac{\alpha_i^2}{\|\boldsymbol{\alpha}\|_2^2}\right\}_{i \in [n]}$. Denote by $\omega:\mathbb{N} \to \mathbb{N}$ the index map corresponding to $S$, meaning that the $i^{\text{th}}$ row of $S$ is $\boldsymbol{e}_{\omega_i}$. For $\zeta > 0$, suppose that
\begin{displaymath}
m \gtrsim \|\boldsymbol{\alpha}\|_2^2 \left(\log(2r+1)+2\log(N) + \log\left(\frac{2}{\zeta}\right)\right).
\end{displaymath}
Then the following holds.
For any $\boldsymbol{x}_0  \in \field^n$ with $\varepsilon, \hat{\boldsymbol{x}}, \boldsymbol{x}^{\perp}$ as in~\Cref{loc:setup_of_signal_recovery_with_subsampled_unitary_measurements_and_gaussian_noise.for_denoising_paper}, and $\hat{\hat{\boldsymbol{x}}}:= \min(1, 1 / \lVert \hat{\boldsymbol{x}}\rVert_{2}) \hat{\boldsymbol{x}}$, we have that
\begin{displaymath}
\mathbb{E} \lVert \hat{\hat{\boldsymbol{x}}} - \boldsymbol{x}_0\rVert_2
\lesssim \frac{\sigma }{ \sqrt{m} } \|\boldsymbol{\alpha}\|_2 \sum_{i=1}^m \frac{1}{\sqrt{n} \alpha_{\omega_i}} + \|\boldsymbol{x}^\perp\|_2 + \varepsilon + \sqrt{\zeta}.
\end{displaymath}
\end{corollary}
The proof of \Cref{loc:noise_corollary_from_deterministic_results_on_union_of_subspaces.statement} consists of translating~\cite[Theorem 3.6]{adcockUnifiedFrameworkLearning2023} into our notation, converting between models of signal acquisition, and randomizing the deterministic noise.
We provide~\hyperlink{loc:noise_corollary_from_deterministic_results_on_union_of_subspaces.proof}{the proof} below. 

Recall that in \Cref{loc:optimized_with:replacement_cs_on_union_of_subspaces.statement} we found the following error bound on the signal recovery error.
\begin{align*}
\lVert \hat{\boldsymbol{x}}- \boldsymbol{x}_0\rVert_2 \leq 9\frac{ \sigma}{\sqrt{ m }} \lVert \boldsymbol{\alpha}\rVert_{2} \sqrt{\min\left(\frac{5}{4\delta}, \frac{1}{n \min(\boldsymbol{\alpha})^2}\right)} \left( \sqrt{ 2 r+1 } + \sqrt{2 \log N} + \sqrt{ \log \frac{20}{\delta}}\right)&\\
+\lVert \boldsymbol{x}^\perp\rVert + 6\lVert SDF\boldsymbol{x}^\perp\rVert_{2}  + \frac{3}{2}\sqrt{ \varepsilon }.&
\end{align*}
We see that our denoising noise error
term exhibits a denoising behavior in its factor of $\frac{1}{ \sqrt{m} }$,
whereas this dependence is cancelled by the sum over $m$ terms in
\Cref{loc:noise_corollary_from_deterministic_results_on_union_of_subspaces.statement}.
%
\begin{proof}[\hypertarget{loc:noise_corollary_from_deterministic_results_on_union_of_subspaces.proof}Proof of \Cref{loc:noise_corollary_from_deterministic_results_on_union_of_subspaces.statement}]

The proof mainly follows from \cite[Theorem 3.6]{adcockUnifiedFrameworkLearning2023},
specifically assumption c). 
From \cite[Equation 5.6]{adcockUnifiedFrameworkLearning2023} we find that
the ``variation" in this theorem is the square of our complexity
function (\Cref{loc:with:replacement_complexity.statement}), which becomes
$\|\boldsymbol{\alpha}\|_2^2$ for the optimized probability vector
$\boldsymbol{p}'$.
From this we find the specified sample complexity.

The correct measurement setup is in \cite[Equation 1.2]{adcockUnifiedFrameworkLearning2023},
\cite[Example 2.4]{adcockUnifiedFrameworkLearning2023}, and \cite[Example 2.5]{adcockUnifiedFrameworkLearning2023},
and see also \cite[Example 2.4]{adcockUnifiedFrameworkLearning2023} for the measurement setup and \cite[Definition 3.1]{adcockUnifiedFrameworkLearning2023} for the definition of the approximate minimization. 

With $\boldsymbol{d} = \{\frac{1}{\sqrt{n p'_i}}\}_{i  \in [n]}$, let $D = \Diag(\boldsymbol{d})$ and $\widetilde{D} = \Diag(S \boldsymbol{d})$. Their measurement model is $\bar{\boldsymbol{b}} = SDF \boldsymbol{x}_0 + \boldsymbol{e}$ for deterministic noise $\boldsymbol{e} \in \mathbb{C}^m$. To convert their result to our measurement model, we multiply both sides by $\widetilde{D}^{-1}$ and substitute in the variables $\boldsymbol{b}= \widetilde{D}^{-1} \bar{\boldsymbol{b}}$ and $\boldsymbol{\eta}= \widetilde{D}^{-1} \boldsymbol{e}$.

We apply their result to the Gaussian noise setting by letting $\boldsymbol{\eta} \overset{\text{iid}}{\sim} \mathcal{N}\left( 0, \frac{\sigma^2}{m}\right)$. Then we find that the noise sensitivity term in~\cite[Corollary 5.3]{adcockUnifiedFrameworkLearning2023} is 
\begin{displaymath}
\mathbb{E} \lVert \widetilde{D} \boldsymbol{\eta}\rVert_2^2= \frac{\sigma^2}{m}\lVert \widetilde{D}\rVert_F^2= \frac{\sigma^2}{m}\sum_{i=1}^m \frac{1}{n p'_{\omega_i}}.
\end{displaymath}
Finally, we take the square root of both sides, and the result follows by Jensen's
inequality on the l.h.s. and distributing the square root (which maintains the inequality) on the r.h.s.
\end{proof}
\end{document}